\documentclass{article}

\usepackage{microtype,color}
\usepackage{graphicx}
\usepackage{subfigure}
\usepackage{booktabs} 
\usepackage{etex}

\usepackage[accepted]{icml2021}

\usepackage{amsmath,amsthm,amssymb,color,multicol,graphicx,float,xcolor,hyperref, stmaryrd,calc, aliascnt, hyperref, cleveref}
\usepackage[english]{babel}
\usepackage[utf8]{inputenc}
\usepackage[inline]{enumitem}
\usepackage{upgreek}
\usepackage{mathrsfs}
\usepackage{amssymb}
\usepackage[english]{babel}
\usepackage{graphicx}
\usepackage{amsfonts}
\usepackage{mathrsfs}
\usepackage{verbatim}
\usepackage{aliascnt}
\usepackage[disable]{todonotes}
\usepackage{xargs}
\usepackage{cellspace}

\usepackage{dsfont}

\crefname{theorem}{theorem}{Theorems}
\Crefname{Theorem}{Theorem}{Theorems}

\newaliascnt{lemma}{theorem}
\newtheorem{lemma}[lemma]{Lemma}
\aliascntresetthe{lemma}
\crefname{lemma}{lemma}{lemmas}
\Crefname{Lemma}{Lemma}{Lemmas}

\newaliascnt{corollary}{theorem}

\aliascntresetthe{corollary}
\crefname{corollary}{corollary}{corollaries}
\Crefname{Corollary}{Corollary}{Corollaries}

\newaliascnt{proposition}{theorem}
\newtheorem{proposition}[proposition]{Proposition}
\aliascntresetthe{proposition}
\crefname{proposition}{proposition}{propositions}
\Crefname{Proposition}{Proposition}{Propositions}

\newaliascnt{definition}{theorem}

\aliascntresetthe{definition}
\crefname{definition}{definition}{definitions}
\Crefname{Definition}{Definition}{Definitions}

\newaliascnt{definitionProposition}{theorem}

\aliascntresetthe{definitionProposition}
\crefname{Proposition and Definition}{Proposition and Definition}{Proposition and Definition}
\Crefname{Proposition and Definition}{Proposition and Definition}{Proposition and Definition}

\newaliascnt{remark}{theorem}

\aliascntresetthe{remark}
\crefname{remark}{remark}{remarks}
\Crefname{Remark}{Remark}{Remarks}

\crefname{example}{example}{examples}
\Crefname{Example}{Example}{Examples}

\crefname{figure}{figure}{figures}
\Crefname{Figure}{Figure}{Figures}

\Crefname{assumption}{\textbf{H}\hspace{-3pt}}{\textbf{H}\hspace{-3pt}}
\crefname{assumption}{\textbf{H}}{\textbf{H}}

\def\normconst{p_\theta(x)}
\def\hatnormconst{\hat{p}_\theta(x)}

\newcommandx{\propmap}[1 ]{\mathrm{T}_{#1}}
\newcommandx{\chunk}[3]{#1_{#2:#3}}

\def\unnormextendedback{\mathrm{P}_\theta^{\operatorname{un}}}

\def\extendedfor{\mathrm{Q}_\phi}
\newcommandx{\Circ}[2]{\underset{#1}{\overset{#2}{\bigcirc}}}
\newcommandx{\Circtext}[2]{{\bigcirc}_{#1}^{#2}}

\def\densinnovationk{\varphi_{d, K}}
\def\densinnovationkp{\varphi_{d, K+1}}

\def\unnormdistr{\gamma}

\def\borel{\mathcal{B}}
\def\ais{\operatorname{AIS}}
\def\AIS{\ais}
\def\elbo{\mathcal{L}}
\def\elboais{\mathcal{L}_{\AIS}}
\def\elbosmc{\elbo_{\mathrm{SIS}}}
\def\densgauss{\mathrm{g}}

\def\initdistr{q_\phi}

\def\RWM{\operatorname{RWM}}
\def\densinnovation{\varphi_d}

\newcommandx\diag[1]{\operatorname{diag}\left(#1\right)}

\def\Normal{\mathrm{N}}

\def\KL{\mathrm{KL}}

\def\msa{\mathsf{A}}

\def\msu{\mathsf{U}}

\newcommand{\mcb}[1]{\mathcal{B}(#1)}

\def\rset{\mathbb{R}}

\def\nsets{\mathbb{N}^*}

\def\rmd{\mathrm{d}}

\def\rmc{\mathrm{C}}
\def\rmC{\mathrm{C}}

\def\Xset{\mathsf{X}}
\def\Xsigma{\mathcal{X}}

\newcommandx{\functionspace}[2][1=+]{\mathbb{F}_{#1}(#2)}

\newcommand{\1}{\mathds{1}}

\newcommand{\LeftEqNo}{\let\veqno\@@leqno}

\newcommand{\N}{\ensuremath{\mathbb{N}}}

\newcommandx{\Vnorm}[2][1=V]{\| #2 \|_{#1}}
\newcommandx{\VnormEq}[2][1=V]{\left\| #2 \right\|_{#1}}
\newcommandx{\norm}[2][1=]{\ifthenelse{\equal{#1}{}}{\left\Vert #2 \right\Vert}{\left\Vert #2 \right\Vert^{#1}}}
\newcommandx{\normLigne}[2][1=]{\ifthenelse{\equal{#1}{}}{\Vert #2 \Vert}{\Vert #2\Vert^{#1}}}

\newcommandx\probaMarkovTilde[2][2=]
{\ifthenelse{\equal{#2}{}}{{\widetilde{\mathbb{P}}_{#1}}}{\widetilde{\mathbb{P}}_{#1}\left[ #2\right]}}

\def\ie{\textit{i.e.}}

\def\eqsp{\;}

\newcommand{\ccint}[1]{\left[#1\right]}

\newcommand{\indi}[1]{\1_{#1}}
\newcommandx{\weight}[2][2=n]{\omega_{#1,#2}^N}

\def\lipcon{\mathrm{L}}

\newcommandx\sequence[3][2=,3=]
{\ifthenelse{\equal{#3}{}}{\ensuremath{\{ #1_{#2}\}}}{\ensuremath{\{ #1_{#2}, \eqsp #2 \in #3 \}}}}
\newcommandx\sequenceD[3][2=,3=]
{\ifthenelse{\equal{#3}{}}{\ensuremath{\{ #1_{#2}\}}}{\ensuremath{( #1)_{ #2 \in #3} }}}

\newcommandx{\sequencen}[2][2=n\in\N]{\ensuremath{\{ #1_n, \eqsp #2 \}}}
\newcommandx\sequenceDouble[4][3=,4=]
{\ifthenelse{\equal{#3}{}}{\ensuremath{\{ (#1_{#3},#2_{#3}) \}}}{\ensuremath{\{  (#1_{#3},#2_{#3}), \eqsp #3 \in #4 \}}}}
\newcommandx{\sequencenDouble}[3][3=n\in\N]{\ensuremath{\{ (#1_{n},#2_{n}), \eqsp #3 \}}}

\newcommand{\wrt}{w.r.t.}

\def\iid{\text{i.i.d.}}

\def\eg{e.g.}

\newcommand{\opnorm}[1]{{\left\vert\kern-0.25ex\left\vert\kern-0.25ex\left\vert #1
    \right\vert\kern-0.25ex\right\vert\kern-0.25ex\right\vert}}

\def\Id{\operatorname{Id}}

\newcommandx{\CPE}[3][1=]{{\mathbb E}_{#1}\left[#2 \left \vert #3 \right. \right]} 
\newcommandx{\CPVar}[3][1=]{\mathrm{Var}^{#3}_{#1}\left\{ #2 \right\}}
\newcommand{\CPP}[3][]
{\ifthenelse{\equal{#1}{}}{{\mathbb P}\left(\left. #2 \, \right| #3 \right)}{{\mathbb P}_{#1}\left(\left. #2 \, \right | #3 \right)}}

\newcommandx{\osc}[2][1=]{\mathrm{osc}_{#1}(#2)}

\def\Id{\operatorname{Id}}

\def\w{w}

\def\Jac{\mathrm{J}}

\newcommand{\ensemble}[2]{\left\{#1\,:\eqsp #2\right\}}

\newcommand{\set}[2]{\ensemble{#1}{#2}}

\newcommand\coupling[2]{\Gamma(\mu,\nu)}

\renewcommand{\leq}{\leqslant}

\def\MALA{\mathrm{MALA}}

\def\Ber{\operatorname{Ber}}

\def\stepsize{\eta} 

\icmltitlerunning{Monte Carlo VAE}

\begin{document}

\twocolumn[
\icmltitle{Monte Carlo Variational Auto-Encoders}




\icmlsetsymbol{equal}{*}

\begin{icmlauthorlist}
  \icmlauthor{Achille Thin}{ecole}
  \icmlauthor{Nikita Kotelevskii}{sk}
  \icmlauthor{Alain Durmus}{ens} 
  \icmlauthor{Maxim Panov}{sk}
  \icmlauthor{Eric Moulines}{ecole,hse,lagrange}
  \icmlauthor{Arnaud Doucet}{oxford}
\end{icmlauthorlist}

\icmlaffiliation{ecole}{CMAP, Ecole Polytechnique, Universite Paris-Saclay, France}

\icmlaffiliation{hse}{HDI Lab, HSE University, Moscow, Russia}

\icmlaffiliation{sk}{CDISE, Skolkovo Institute of Science and Technology, Moscow, Russia}

\icmlaffiliation{ens}{Ecole Nationale Sup\'erieure Paris-Saclay, France}

\icmlaffiliation{oxford}{University of Oxford}

\icmlaffiliation{lagrange}{Centre de recherche Lagrange en mathematiques et calcul}

\icmlcorrespondingauthor{Achille Thin}{achille.thin@polytechnique.edu}

\icmlkeywords{Machine Learning, ICML}

\vskip 0.3in
]



\printAffiliationsAndNotice{\icmlEqualContribution} 

\begin{abstract}
  Variational auto-encoders (VAE) are popular deep latent variable models which are trained by maximizing an Evidence Lower Bound (ELBO). To obtain tighter ELBO and hence better variational approximations, it has been proposed to use importance sampling to get a lower variance estimate of the evidence. 
  However, importance sampling is known to perform poorly in high dimensions. While it has been suggested many times in the literature to use more sophisticated algorithms such as Annealed Importance Sampling (AIS) and its Sequential Importance Sampling (SIS) extensions, the potential benefits brought by these advanced techniques have never been realized for VAE: the AIS estimate cannot be easily differentiated, while SIS requires the specification of carefully chosen backward Markov kernels.
  In this paper, we address both issues and demonstrate the performance of the resulting Monte Carlo VAEs on a variety of applications.
\end{abstract}

\section{Introduction}
  Variational Auto-Encoders (VAE) introduced by~\citep{kingma:welling:2014} are a very popular class of methods in unsupervised learning and generative modelling.  
  These methods aim at finding a parameter $\theta$ maximizing the marginal log-likelihood $p_\theta(x) = \int p_\theta(x,z)\rmd z$ where $x\in\rset^N$ is the observation and $z\in\rset^d$ is the latent variable. They rely on the introduction of an additional parameter $\phi$ and a family  of variational distributions $q_\phi(z| x)$. The joint parameters $\{\theta, \phi\}$  are then inferred through the optimization of the Evidence Lower Bound (ELBO) defined as
  \begin{multline}
  \nonumber
    \elbo(\theta, \phi)= 
    \int \log\left(\frac{p_\theta(x, z)}{q_\phi(z | x)}\right) q_\phi(z | x) \rmd z\eqsp\\
     =\log p_\theta(x) - \KL\bigl(q_\phi(z | x) ~\|~ p_\theta(z | x)\bigr) \leq \log p_\theta(x) \eqsp.
  \end{multline}
  The design of expressive variational families has been the topic of many works and is a core ingredient in the efficiency of VAE~\cite{rezende2015variational,kingma2016improving}.
  Another line of research consists in using positive unbiased estimators $\hat{p}_\theta(x)$ of the loglikelihood $p_\theta(x)$ for $q_{\phi}$, \ie~$\mathbb{E}_{q_\phi}[\hat{p}_\theta(x)] = p_\theta(x)$. Indeed, as noted in~\cite{mnih2016variational}, it follows from Jensen's inequality that 
  \begin{align}\label{eq:generalELBO}
    \elbo(\theta, \phi) &= \mathbb{E}_{q_\phi}[\log \hat{p}_\theta(x)]\leq \log p_\theta(x)\eqsp.
  \end{align}
  A Taylor expansion shows that
  \begin{align*}
    \elbo(\theta, \phi) \approx \log p_\theta(x) - \frac{1}{2}\textup{var}_{q_\phi} \left[\frac{\hat{p}_\theta(x)}{p_\theta(x)}\right] \eqsp;
  \end{align*}
  see e.g. \cite{maddison2017filtering,domke2018importance} for formal results. Hence the ELBO becomes tighter as the variance of the estimator decreases.

  A common method to obtain an unbiased estimate is built on importance sampling; \ie~$\hatnormconst = n^{-1}\sum_{i=1}^n [p_\theta(x,z_i)/\initdistr(z_i| x)]$ for $z_i\overset{\textup{i.i.d.}}{\sim} q_{\phi}(\cdot|x)$. In particular, combined with~\eqref{eq:generalELBO}, we obtain the popular Importance Weighted Auto Encoder (IWAE) proposed by~\cite{burda:grosse:2015}. However, it is expected that the relative variance of this importance-sampling based estimator typically increases with the dimension of the latent $z$. To circumvent this issue, we suggest in this paper to consider other estimates of the evidence which have shown great success in the Monte Carlo literature.
  In particular, Annealed Importance Sampling (AIS)~\cite{neal2001annealed,wu:burda:grosse:2016}, and its Sequential Importance Sampling (SIS) extensions~\cite{Del-Moral:2006} define state-of-the-art estimators of the evidence. These algorithms rely on an extended target distribution for which an efficient importance distribution can be defined using non-homogeneous Markov kernels.
 

  It has been suggested in various contributions that AIS could be useful to train VAE~\cite{salimans:kingma:welling:2015,wu:burda:grosse:2016,maddison2017filtering,wunoe2020stochastic}. However, to the authors knowledge, no contribution discusses how an unbiased gradient of the resulting ELBO could be obtained. Indeed, the main difficulty in this computation arises from the MCMC transitions in AIS. 
  As a result, when this estimator is used, alternatives to the ELBO~\eqref{eq:generalELBO} have often been considered, see e.g.~\cite{ding2019learning}.

  Whereas AIS requires using MCMC transition kernels, SIS is more flexible and can exploit Markov transition kernels which are only approximately invariant w.r.t. to a given target distribution, \eg~unadjusted Langevin kernels. In this case, the construction of the augmented target distribution, which is at the core of the estimator, requires the careful selection of a class of auxiliary `backward' Markov kernels. \citep{salimans:kingma:welling:2015,ranganath2016hierarchical, maaloe2016auxiliary,goyal2017variational,huang:tan:lacoste:courville:2018} propose to learn such auxiliary kernels parameterized with neural networks through the ELBO. However, as illustrated in our simulations, this comes at an increase in the overall computational cost.

\paragraph{Our contributions.}
  The contributions of this paper are as follows:
  \begin{enumerate}[label=\textbf{(\roman*)}]
    \item 
    We show how to obtain an SIS-based ELBO relying on undadjusted Langevin dynamics which, contrary to~\citep{salimans:kingma:welling:2015,goyal2017variational,huang:tan:lacoste:courville:2018}, does not require introducing and optimizing backward Markov kernels. In addition, an unbiased gradient estimate of the resulting ELBO which exploits the reparameterization trick is derived.
    \item 
    We propose an unbiased gradient estimate of the ELBO computed based on AIS. At the core of this estimate is a non-standard representation of Metropolis--Hastings type kernels which allows us to differentiate them. This is combined to a variance reduction technique for improved efficiency.
    \item 
    We apply these new methods to build novel Monte Carlo VAEs, and show their efficiency on real-world datasets.
  \end{enumerate}

All the theoretical results are detailed in the supplementary material.
\section{Variational Inference via Sequential Importance Sampling}
\label{sec:diffSIS}
\subsection{SIS estimator}
\label{subsec:sis_est}
  The design of efficient proposal importance distributions has been proposed in~\cite{crooks1998nonequilibrium,neal2001annealed} starting from an annealing schedule, and was later extended by~\cite{Del-Moral:2006}.
  Let $\{\unnormdistr_k\}_{k = 0}^K$ be a sequence of unnormalized ``bridge'' densities satisfying $\unnormdistr_0(z) = \initdistr(z| x)$, $\unnormdistr_K(z)= p_\theta(x, z)$. Note that $\unnormdistr_K$ is not normalized, in contrast to  $\unnormdistr_0$.
  Here we set 
  \begin{equation}
      \unnormdistr_{k}(z)  = q_\phi(z | x)^{1 - \beta_k} p_\theta(x, z)^{\beta_k}
  \end{equation}
  for an annealing schedule $0 = \beta_0 < \dots < \beta_{K} = 1$, but alternative sequences of intermediate densities could be considered.
  At each iteration, SIS samples a non-homogeneous Markov chain use the transition kernels  $\{ M_k \}_{k=1}^K$. In this section, we assume that $M_k$ admits a positive transition density $ m_k$, such that $m_k$  leaves $\unnormdistr_k$ invariant, \ie~ $\int\unnormdistr_k(z)m_k(z, z')\rmd z = \unnormdistr_k(z')$, or is only approximately invariant. In particular, $m_k$ typically depends on the data $x$. However, to simplify notation, this dependence is omitted. 
  
  Based on this sequence of transition densities, SIS considers the joint density on the path space $\chunk{z}{0}{K} \in \rset^{d{(K + 1)}}$
  \begin{equation}\label{eq:extendedproposal}
    \initdistr^K(\chunk{z}{0}{K} | x) = \initdistr(z_0 | x) \prod_{k=1}^K m_k(z_{k-1}, z_k)\eqsp,
  \end{equation}
  where $\chunk{z}{i}{j} = (z_i, \dots, z_j)$ for $0 \leq i \leq j$. By construction, we expect the $K$-th marginal $\initdistr^K(z_K| x)= \int\initdistr^K(\chunk{z}{0}{K}| x)\rmd \chunk{z}{0}{K-1}$ to be close to $p_\theta(z_K|x)$. However, we cannot use importance sampling to correct for the discrepancy between these two densities, as $\initdistr^K(z_K| x)$ is typically intractable. To bypass this problem, we introduce another density on the path space 
  \begin{equation}\label{eq:extendedtarget}
    p_\theta^K(x, \chunk{z}{0}{K}) = p_\theta(x, z_K) \prod_{k=K-1}^0 \ell_{k}(z_{k+1},z_{k})\eqsp,
  \end{equation}
  where $\{\ell_k\}_{k = 0}^{K-1}$ is a sequence of auxiliary positive transition densities. Note that in this case, the $K$-th marginal  $ \int p_\theta^K(x, \chunk{z}{0}{K}) \rmd \chunk{z}{0}{K-1}$ is exactly  $p_\theta(x, z_K) $. 
  Using now importance sampling on the path space, we obtain the following unbiased SIS estimator~\cite{Del-Moral:2006,salimans:kingma:welling:2015} by sampling independently $\chunk{z}{0}{K}^i \sim \initdistr^K(\cdot | x)$ and computing
  \begin{equation}
  \label{eq:SiS-estimator-Z}
    \hatnormconst = \frac{1}{n}\sum_{i = 1}^n w^K(\chunk{z}{0}{K}^{i})\eqsp,\eqsp w^K(z_{0:K})=\frac{p_\theta^K(x, \chunk{z}{0}{K}) }{\initdistr^K(\chunk{z}{0}{K}| x)}.
  \end{equation}
\subsection{AIS estimator}
\label{subsec:ais_est}
  The selection of the kernel $\{\ell_k\}_{k = 0}^{K-1}$ has a large impact on the variance of the estimator. The optimal reverse kernels $\ell_k$ minimizing the variance of $\hatnormconst$ are given by 
  \begin{equation}
  \label{eq:optimalbackward}
    \ell_{k-1}(z_k, z_{k-1})=\frac{q_{\phi, k-1}(z_{k-1}\vert x)m_k(z_{k-1}, z_k)}{q_{\phi,k}(z_k\vert x)}\eqsp,
  \end{equation}
  where $q_{\phi,k}(z_k| x)= \int q_\phi^K(\chunk{z}{0}{K}|x)\rmd \chunk{z}{0}{K}^{-k}$ is the $k$-th marginal of $\initdistr^K$; see~\cite{Del-Moral:2006}. However, the resulting estimator $\hatnormconst $ is usually intractable. 
   An approximation to~\eqref{eq:optimalbackward} leading to a tractable estimator is provided by AIS~\cite{crooks1998nonequilibrium,neal2001annealed}. When $m_k$ is $\unnormdistr_k$-invariant, then, by assuming that $q_{\phi,k} \approx q_{\phi, k-1} \approx \gamma_k$ in~\eqref{eq:optimalbackward}, we obtain
  \begin{equation}
  \label{eq:reversal}
    \ell_{k-1}(z_k, z_{k-1})=\frac{ \unnormdistr_k(z_{k-1}) m_k(z_{k-1}, z_k)}{\unnormdistr_k(z_{k})}\eqsp.
  \end{equation}
 We refer to this kernel as the \emph{reversal} of $m_k$. In particular, if $m_k$ is $\unnormdistr_k$-reversible,
 \ie~$\unnormdistr_k(z_{k-1}) m_k(z_{k-1}, z_k)= \unnormdistr_k(z_{k}) m_k(z_{k}, z_{k-1})$, then $\ell_{k-1}=m_k$. 
 Note that the weights~\eqref{eq:SiS-estimator-Z} can be computed using the decomposition $w^K(z_{0:K})=\prod_{k = 1}^K w_{k} (z_{k-1}, z_k)$ with 
  \begin{equation}\label{eq:incweight}
    w_{k} (z_{k-1}, z_k)=\frac{\unnormdistr_k(z_k) \ell_{k-1}(z_{k}, z_{k-1})}{\unnormdistr_{k-1}( z_{k-1})m_k(z_{k-1}, z_k)} \eqsp,
  \end{equation}
 which simplifies when using~\eqref{eq:reversal} as
  \begin{equation}
  \label{eq:AISweight}
    w_{k} (z_{k-1}, z_k) = \unnormdistr_k(z_{k - 1}) / \unnormdistr_{k-1}(z_{k - 1})\eqsp.
  \end{equation}
 In contrast to previous works, we will consider in the next section transition kernels $m_k$ which are only approximately $\unnormdistr_k$-invariant but still build $\ell_{k-1}$ in the spirit of~\eqref{eq:reversal}.

\subsection{SIS-ELBO using unadjusted Langevin}
\label{subsec:LangevinELBO}
For any $k$, we consider the Langevin dynamics associated to $\unnormdistr_k$ and the corresponding Euler-Maruyama discretization.
Then, we choose for $m_k$ the transition density associated to this discretization 
\begin{equation}
    \label{eq:langevin_density}
    m_k(z_{k - 1}, z_k) = \Normal(z_k; z_{k - 1} + \stepsize\nabla\log\unnormdistr_k(z_{k - 1}), 2\stepsize \Id )\eqsp.
\end{equation}
Note that sampling $\initdistr^K$ boils down to sampling $Z_0\sim \initdistr$ and defining the Markov chain $\{Z_k\}_{k \leq K}$ recursively by
  \begin{equation}
  \label{eq:langevin_eq}
    Z_{k} = Z_{k - 1} + \stepsize \nabla\log\unnormdistr_k(Z_{k - 1}) + \sqrt{2 \stepsize} U_{k}\eqsp,
  \end{equation}
  where $U_{k} \sim \Normal(0, \Id_d)$. 
  Moreover, as the continuous Langevin dynamics is reversible \wrt\ $\unnormdistr_k$, this suggests that we can approximate the reversal $\ell_{k-1}$ of $m_k$ by $m_k$ directly as in AIS and thus select for any $k$, $\ell_{k-1}=m_k$ as done in ~\cite{heng2017controlled}. However, in this case, the weights $w_{k}(z_{k-1},z_k)$ do not simplify to $\unnormdistr_k(z_{k-1})/\unnormdistr_{k-1}(z_{k-1})$ as $m_k$ is not exactly $\unnormdistr_k$-invariant and we need to rely on the general expression given in~\eqref{eq:incweight}. This approach was concurrently and independently proposed in \cite{wunoe2020stochastic} but they do not discuss gradient computations therein.
  
  Based on~\eqref{eq:generalELBO} and~\eqref{eq:SiS-estimator-Z}, we introduce
  \begin{equation}
  \label{eq:elbo_sis_1}
    \elbosmc =   \int_{} \log \left(\prod_{k = 1}^K w_{k} (z_{k - 1}, z_k)\right) \initdistr^K(\chunk{z}{0}{K}| x) \rmd \chunk{z}{0}{K}\eqsp.
    \end{equation}
  We consider a reparameterization of~\eqref{eq:elbo_sis_1} based on the Langevin mappings associated with target $\unnormdistr_k$ given by 
  \begin{multline}
  \label{eq:langevin_map}
    \propmap{k, u}(z_{k - 1}) = z_{k - 1} + \stepsize \nabla\log\unnormdistr_k(z_{k - 1}) + \sqrt{2 \stepsize} u\eqsp.
  \end{multline}
  %
  An easy change of variable based on the identity $Z_k = \propmap{k, U_k}(Z_{k-1})$ in~\eqref{eq:langevin_eq} shows that $\elbosmc$ can be written as
  \[
  \int_{} \log\left( \prod_{k=1}^{K}  w_{k}(z_{k-1},z_k)\right)\initdistr(z_0| x) \densinnovationk(\chunk{u}{1}{K})\rmd z_0\rmd \chunk{u}{1}{K},
  \]
  where $\densinnovation$ stands for the density of the $d$-dimensional standard Gaussian distribution, $\densinnovationk(\chunk{u}{1}{K}) = \prod_{i=1}^K \densinnovation(u_i)$,  and we write $z_k = \propmap{k, u_k}\circ\dots\circ \propmap{1, u_1}(z_0) = \Circtext{i = 1}{k} \propmap{i, u_i} (z_0)$. 
  By~\eqref{eq:SiS-estimator-Z} and as $\ell_{k-1}=m_k$, our objective thus finally writes
  \begin{align}
    \label{eq:elbo_sis}
    \elbosmc&(\theta, \phi; x) = \int q_\phi(z_0| x) \densinnovationk(\chunk{u}{1}{K}) \\
    \nonumber
   \times &\log\left(\frac{p_\theta(x, z_K)\prod_{k=1}^K m_k(z_k, z_{k-1})}{q_\phi(z_0 | x) \prod_{k = 1}^K m_k(z_{k - 1}, z_k)}\right) \rmd z_0 \rmd \chunk{u}{1}{K}\eqsp.
  \end{align}
  This defines the reparameterizable Langevin Monte Carlo VAE (L-MCVAE). The algorithm to obtain an unbiased SIS estimate of $p_\theta(x)$ is described in \Cref{alg:differentiable_sis_flows}. This estimate is related to the one presented in~\cite{caterini:doucet:2018}, however this work builds on a deterministic dynamics which limits the expressivity of the induced variational approximation. In contrast, we rely here on a stochastic dynamics. While we limit ourselves here to undajusted (overdamped) Langevin dynamics, this could also be easily extended to unajusted underdamped Langevin dynamics \cite{monmarche2020high}.
\begin{algorithm}[!t]
    \caption{Langevin Monte Carlo VAE}
    \label{alg:differentiable_sis_flows}
    \begin{algorithmic}
      \STATE {\bfseries Input:}  Number of steps $K$, 
      initial distribution $\initdistr$, unnormalized target distribution $p_\theta$, step-size $\eta$,  annealing schedule $\{\beta_k\}_{k = 0}^{K}$.
       \STATE {\bfseries Output:} SIS estimator $W$ of $\log p_\theta(x)$.
        \\\hrulefill
        \STATE Draw $z_0\sim \initdistr(\cdot | x)$;
        \STATE Set $W = -\log \initdistr(z_0 | x)$;
      \FOR{$k=1$ {\bfseries to} $K$}
      \STATE Draw $u_k\sim \densinnovation$;
      \STATE Set $\unnormdistr_k(\cdot) = \beta_k \log p_\theta (x, \cdot) + (1-\beta_k)\log \initdistr(\cdot| x) $;
      \STATE Set $z_k = z_{k - 1} + \stepsize \nabla \log \unnormdistr_k(z_k)  + \sqrt{2\stepsize}u_k$;
      \STATE Set $W= W + \log m_k(z_{k}, z_{k - 1}) - \log m_k(z_{k - 1}, z_k)$;
      \ENDFOR
      \STATE Set $W= W + \log p_\theta (x, z_K)$;
      \STATE Return $W$
    \end{algorithmic}
  \end{algorithm}

  \section{Variational Inference via Annealed Importance Sampling}
  \label{sec:ais}
In Section \ref{subsec:LangevinELBO}, we derived the SIS estimate of the evidence computed using ULA kernels $M_k$, whose invariant distribution approximates $\unnormdistr_k$. We can differentiate the resulting ELBO and exploit the reparametrization trick as these kernels admit a density $m_k$  w.r.t. Lebesgue measure. When computing the AIS estimates, we need Markov kernels $M_k$ that are $\unnormdistr_k$-invariant. Such Markov kernels  $M_k$ most often rely on a Metropolis--Hastings (MH) step and therefore do not typically admit a transition density.  While this does not invalidate the expression of the AIS estimate presented earlier, it complicates significantly the computation of an unbiased gradient estimate of the corresponding ELBO. In this section, we propose a way to compute an unbiased estimator of the ELBO gradient for MH Markov kernels. We use here elementary measure-theoretical notations  which are recalled in \Cref{sec:notations-supp}

\subsection{Differentiating Markov kernels}
\paragraph{Markov kernel.} Let $\borel(\rset^{d})$ denote the Borelian $\sigma$-field associated to $\rset^d$. A Markov kernel $M$ is a function defined on $\rset^{d} \times \borel( \rset^{d})$,
such that for any $z \in \rset^{d}$, $M(z, \cdot)$ is a probability distribution on $\borel(\rset^{d})$, \ie\ $M(z, A)$ is the probability starting from $z$ to hit the set $A\subset\rset^d$.
 The simplest is ``deterministic'', in which case  $Q(z,A)= \updelta_{\propmap{}(z)}(A)$, where $\propmap{}$ is a measurable mapping on $\rset^{d}$ and $\updelta_y$ is the Dirac mass at $y$. Instead of a single mapping $\propmap{}$, we can consider a family of ``indexed'' mappings $\set{\propmap{u}}{u \in \rset^{d_u}}$. If $\densgauss$ is a p.d.f on $\rset^{d_U}$, we consider $$M(z,A)= \int \indi{A}(\propmap{u}(z)) \densgauss(u) \rmd u\eqsp.$$ 
    To sample $M(z,\cdot)$, we first sample $u \sim \densgauss$ and then apply the mapping $\propmap{u}(z)$. If $d=d_u$, we consider the function $G_z: \rset^d \mapsto \rset^d$ defined for all $z \in \rset^d$ by 
  \begin{equation}
  \label{eq:definition-G-z}
      G_z: u \mapsto G_z(u) =\propmap{u}(z)
  \end{equation} 
  If $G_z$ is a diffeomorphism, then by applying the change of variables formula, we obtain
  \begin{align}
  \label{eq:markov-kernel-transition}
    M(z, A) &= \textstyle{\int} \indi{A}\bigl(G_z(u)\bigr) \densgauss(u) \rmd u \\ \nonumber
    &= \textstyle{\int} \indi{A}(z') m(z, z') \rmd z' \eqsp,
  \end{align}
  where, denoting $\Jac_{G_z^{-1}}(z')$ is the absolute value of the Jacobian determinant of $G_z^{-1}$ evaluated at $z'$, we set 
  \begin{equation}
  \label{eq:density_proposal}
      m(z, z')= \Jac_{G_z^{-1}}(z') \densgauss\bigl(G_z^{-1}(z')\bigr)\eqsp. 
  \end{equation}
  In this case, the Markov kernel has a transition density $m(z, z')$. This is the setting considered in the previous section.
  
  Finally, some Markov kernels have both a  deterministic and a continuous component. This is for example the case of Metropolis--Hastings (MH) kernels:
  \begin{align}
  \label{eq:mh_kernel}
    &M(z, \msa) = \int_{\msu}  Q_u(z, \msa) \densgauss(u) \rmd u \eqsp,  \text{ where } \\ \nonumber
    &Q_u\bigl(z, \msa\bigr) = \alpha_u(z) \updelta_{\propmap{u}(z)}(\msa) + \bigl\{1 - \alpha_u(z)\bigr\} \updelta_z(\msa) \eqsp,
  \end{align}
  with $\alpha_u(z) = \alpha\bigl(z, \propmap{u}(z)\bigr)$ is the \emph{acceptance function} and $\propmap{u}\colon \rset^d \to \rset^d$ is the \textit{proposal mapping}. In the sequel, we denote $\alpha_u^1(z) = \alpha_u(z)$ and $\alpha^0_u(z) = 1 - \alpha_u(z)$, and set $\propmap{u}^0(z) = z$. With these notations, \eqref{eq:mh_kernel} can be rewritten in a more concise way  
  \begin{equation}
   \label{eq:Q_mh_kernel-concise}
      Q_u(z, \msa) = \sum_{a=0}^1 \alpha^a_u(z) \updelta_{\propmap{u}^a(z)}(\msa) \eqsp.
  \end{equation}  
  To sample $M(z,\cdot)$, we first draw $u \sim \densgauss$ and then compute the proposal $y = \propmap{u}(z)$. With probability $\alpha_u(z)$, the next state of the chain is set to $z'=y$, and $z' = z$ otherwise. If $G_z$ defined in \eqref{eq:definition-G-z} is a diffeomorphism, then the Metropolis--Hasting kernel may be expressed as  
 \begin{multline}
  \nonumber
    M(z,A) = \int \alpha(z,z') m(z,z') \indi{A}(z') \rmd z' \\+ \bigl(1 - \bar{\alpha}(z)\bigr) \delta_z(A),
  \end{multline}
  where   $\bar{\alpha}(z) = \int \alpha(z, z') m(z, z') \rmd z'$ is the mean acceptance probability at $z$ (the probability of accepting a move) and $m(z,z')$ is defined as in \eqref{eq:density_proposal}.
  In MH-algorithms, the acceptance function
  $\alpha\colon \rset^{2d} \to \ccint{0, 1}$ is chosen so that $M$ is $\pi$-reversible $\pi(\rmd z) M(z, \rmd z')= \pi(\rmd z')M(z', \rmd z)$, where $\pi$ is the target distribution. This implies, in particular, that  $M$ is  $\pi$-invariant. Standard MH algorithms use $\alpha(z,z')= 1 \wedge \pi(z') m(z',z) / \pi(z) m(z,z')$; see \cite{tierney:1994}. 
  %

  To illustrate these  definitions and constructions, consider first the symmetric Random Walk Metropolis Algorithm (RWM). In this case, $d_U =d$ and $\densgauss \leftarrow \densinnovation$, where $\densinnovation$ is the $d$-dimensional  standard Gaussian density. The proposal mapping is given by
  \begin{equation}
  \nonumber
    G_z^{\RWM}(u)= \propmap{u}^{\RWM}(z) = z + \Sigma^{1/2} u \eqsp, 
  \end{equation}
  where $\Sigma$ is a positive definite matrix, and the acceptance function is given by $\alpha^{\RWM}_{ u}(z) = 1 \wedge \bigl(\pi\bigl(T^{\RWM}_{ u}(z)\bigr)/\pi(z)\bigr)$.

  Consider now the Metropolis Adjusted Langevin Algorithm (MALA); see~\cite{besag:1994}. Assume that $z \mapsto \log \pi(z)$ is differentiable and denote by $\nabla \log \pi(z)$ its gradient. The Langevin proposal mapping $\propmap{u}^{\MALA}$  
  is defined by 
  \begin{equation}
      \label{eq:langevin_prop}
      G_z^{\MALA} (u) = \propmap{u}^{\MALA}(z) = z + \eta \nabla\log\pi(z) + \sqrt{2\eta} u\eqsp.
  \end{equation}
 We set $\densgauss\leftarrow \densinnovation$ 
  and $ \alpha^{\MALA}_{ u}(z)$ is the acceptance given by
  \begin{equation}
    \label{eq:mala_acceptance}
    \alpha^{\MALA}_{ u}(z) = 1 \wedge \frac{\pi\bigl(\propmap{u}^{\MALA} (z)\bigr) m_\eta\bigl(\propmap{u}^{\MALA} (z), z\bigr)} {\pi(z) m_\eta\bigl(z,\propmap{u}^{\MALA} (z)\bigr)}  \eqsp,
  \end{equation}
  where $m_\stepsize(z,z')= \stepsize^{-1/2} \densinnovation\bigl(\stepsize^{-1/2}\{z' -\propmap{0}^{\MALA}(z)\}\bigr)$, similarly to~\eqref{eq:langevin_density}.
\begin{lemma}
\label{lem:langevin_map_diff}
For all $z\in\rset^d$, $G_z^{\MALA}$ is a $\rmc^1$-diffeomorphism. 
Moreover, assume that $\log\pi$ is $\lipcon$-smooth with $\lipcon>0$, \ie\ for $z,z'\in\rset^d$, $\|\nabla \log\pi(z') - \nabla\log\pi(z) \leq \lipcon \|z'-z\|$. Then, if $0\leq \eta<L^{-1}$, for all $u\in\rset^d$, $\propmap{u}^{\MALA}$ defined in \eqref{eq:langevin_prop} is a $\rmc^1$-diffeomorphism. 
\end{lemma}
The proof of \Cref{lem:langevin_map_diff} is postponed to \Cref{spsec:proofs}.
\subsection{Differentiable AIS-based ELBO}
  We now generalize the derivation of \Cref{subsec:sis_est} to handle the case where $M_k$ and its reversal do not admit transition densities. 
  In this case, the  proposal and unnormalized target distributions are defined by  
  \begin{align}
  \label{eq:definie-extended}
    & \extendedfor^K(\rmd \chunk{z}{0}{K}|x) = \initdistr(z_0| x)\rmd z_0\prod_{k=1}^K M_k(z_{k-1},\rmd z_k)\eqsp,\\
    \nonumber
   \label{eq:define-unnormalized}
    & \unnormextendedback(x, \rmd \chunk{z}{0}{K}) = p_\theta(x, z_K)\rmd z_K \prod_{k=K}^1 L_{k-1}(z_{k}, \rmd z_{k-1}) \eqsp,
  \end{align}
  where we define the reversal Markov kernel  $L_{k-1}$ by 
  \begin{multline*}
    \unnormdistr_k(z_{k-1}) \rmd z_{k-1} M_k(z_{k-1},\rmd z_{k})\\ =\unnormdistr_{k}(z_{k})\rmd z_k L_{k-1}(z_{k},\rmd z_{k-1})  \eqsp.
  \end{multline*}
We consider then the AIS estimator presented in Section~\ref{subsec:sis_est} in~\eqref{eq:SiS-estimator-Z} by sampling independently $\chunk{z}{0}{K}^i \sim \extendedfor^K(\cdot | x)$  and with $w^K(z_{0:K})=\prod_{k = 1}^K w_{k} (z_{k-1}), \, w_{k} (z_{k-1}) = \unnormdistr_k(z_{k - 1}) / \unnormdistr_{k-1}(z_{k - 1})$. A rigorous proof of the unbiasedness of the resulting estimator can be found in the Supplementary Material and is based on the formula
  \begin{equation}
  \label{eq:smcsampler_identity}
    \normconst = \int w^K(\chunk{z}{0}{K}) \extendedfor^K(\rmd \chunk{z}{0}{K}| x) \eqsp.
  \end{equation}
  In the sequel, we use   Metropolis Adjusted Langevin Algorithm (MALA) kernels $\{ M_k \}_{k=1}^K$ targeting $\unnormdistr_k$ for each $k \in \{1,\dots,K\}$. By construction, the Markov kernel $M_k$ is reversible \wrt\ $\unnormdistr_k$ and we set for the reversal kernel $L_{k-1}=M_k$. Note that we could easily generalize to other cases, especially inspired by recent works on non-reversible MCMC algorithms~\cite{thin:kotelevskii:andrieu:moulines:2020}.
  
  For $k \in \{1,\dots,K\}$, we use the representation of MALA kernel $M_k$ outlined in \eqref{eq:mh_kernel} with proposal mapping  $\propmap{k,u}$ and acceptance function $\alpha_{k,u}$ defined as in \eqref{eq:langevin_prop} and \eqref{eq:mala_acceptance} with $\pi \leftarrow \gamma_k$. We set
  \begin{equation}
\label{eq:sum_kernel_Q_a_r}
Q_{k,u}(z,\rmd z')= \sum_{a=0}^1\alpha_{k,u}^a(z) \updelta_{\propmap{k,u}^a(z)}(\rmd z')\eqsp.
  \end{equation}
  By construction, the MALA kernel $M_k$ (see \eqref{eq:mh_kernel}) writes $M_k(z,\rmd z') = \int Q_{k,u}(z,\rmd z') \densinnovation(u) \rmd u $. Plugging this representation into \eqref{eq:definie-extended}, we get
  \begin{multline}
    \nonumber
    \extendedfor^K(\rmd \chunk{z}{0}{K} | x) ={\int_{}} \prod_{k=1}^K Q_{k,u_k}(z_{k - 1}, \rmd z_k) \\ \times \initdistr( z_0| x)  \densinnovationk(\chunk{u}{1}{K})\rmd z_0\rmd\chunk{u}{1}{K} \eqsp,
  \end{multline}
  writing $\densinnovationk(\chunk{u}{1}{K}) = \prod_{i=1}^K\densinnovation(u_i)$.
 Eq.~\eqref{eq:sum_kernel_Q_a_r} suggests to consider the extended distribution on ${( \chunk{z}{0}{K}, \chunk{a}{1}{K}, \chunk{u}{1}{K}})$: 
  \begin{multline*}
    \extendedfor^K({\rmd \chunk{z}{0}{K}, \chunk{a}{1}{K},\rmd \chunk{u}{1}{K}}| x) =  \initdistr(z_0| x) \prod_{k=1}^K \alpha^{a_k}_{k, u_k} \bigl( z_{k - 1}\bigr)
    \\
    \times \prod_{k=1}^K  \updelta_{\propmap{k, u_{k}}^{a_k}(z_{k-1})}(\rmd z_k) \densinnovationk(\chunk{u}{1}{K})\rmd z_0\rmd\chunk{u}{1}{K} \eqsp,
  \end{multline*}
  which admits again as a marginal $\extendedfor^K(\rmd \chunk{z}{0}{K} | x)$. Note that the variables $\chunk{a}{1}{K}$ correspond to the binary outcomes of the $K$ A/R steps. By construction, $\chunk{z}{1}{K}$ are \emph{deterministic functions} of $z_0$, $\chunk{a}{1}{K}$ and $\chunk{u}{1}{K}$: for each $k \in \{1,\dots,K\}$, $z_k= \Circtext{i = 1}{k} \propmap{i, u_i}^{a_i} (z_0)$.
  
  Set $w^K(z_{0:K})=\prod_{k = 1}^K w_{k} (z_{k-1})$,
  and 
  \begin{align*}
      &A(z_0, \chunk{a}{1}{K}, \chunk{u}{1}{K}) = \prod_{k=1}^K \alpha^{a_k}_{k, u_k} \bigl(\Circ{i=1}{k-1} \propmap{i, u_i}^{a_i} (z_0)\bigr)\eqsp,\\
      &W(z_0, \chunk{a}{1}{K}, \chunk{u}{1}{K}) = \log\prod_{k = 1}^{K} w_{k}\Bigl(\Circ{i = 1}{k-1} \propmap{i, u_i}^{a_i} (z_0)\Bigr)\eqsp.
  \end{align*}
  Given $z_0$ and $\chunk{u}{1}{K}$, $A(z_0, \chunk{a}{1}{K}, \chunk{u}{1}{K})$ is the conditional  distribution of the A/R random variables $\chunk{a}{1}{K}$. It is easy to sample this distribution recursively: for each $k \in \{1,\dots,K\}$ we sample $a_k$ from a Bernoulli distribution of parameter $\alpha_{k, u_k}(z_{k-1})$ and we set $z_k= \propmap{k,u_k}^{a_k}(z_{k-1})$.  
  Eqs.~\eqref{eq:SiS-estimator-Z} and~\eqref{eq:AISweight} naturally lead us to consider the ELBO
  \begin{align*}
      \elboais&= \int_{} \sum_{\chunk{a}{1}{K}}  \extendedfor^K({\rmd \chunk{z}{0}{K}, \chunk{a}{1}{K},\rmd \chunk{u}{1}{K}} | x) \log\left[\w^K(z_{0:K})\right] \\
    &= \int \sum_{\chunk{a}{1}{K}}  \initdistr( z_0 | x) A(z_0, \chunk{a}{1}{K}, \chunk{u}{1}{K})   \\
    &\qquad\quad\times W(z_0, \chunk{a}{1}{K}, \chunk{u}{1}{K})  \, \densinnovationk(\chunk{u}{1}{K}) \rmd z_0 \rmd \chunk{u}{1}{K}\eqsp,
  \end{align*}
  where in the last line we have integrated \wrt~$z_{1:K}$. The fact that $\elboais$ is an ELBO stems immediately from \eqref{eq:smcsampler_identity}, by applying Jensen's inequality. We can optimize this ELBO \wrt~the different parameters at stake, even possibly the parameters of the proposal mappings.

We  reparameterize  the  latent variable distribution $q_\phi(z_0|x)$ in terms  of  a  known  base  distribution  and a  differentiable  transformation  (such  as  a  location-scale transformation). Assuming for simplicity that $q_\phi(z_0|x)$ is a Gaussian distribution $\Normal(z_0;\mu_\phi(x),\sigma_\phi^2(x))$, then the location-scale
transformation  using  the  standard  Normal  as  a  base distribution allows us to reparameterize $z_0$ as
$z_0 = \mu_\phi(x) + \sigma_\phi(x) \cdot u_0 = V_{\phi,x}(u_0)$, with $u_0\sim\densinnovation$. 
Using this reparameterization trick, we can write $\nabla\elboais = \nabla{\elboais}_1 + \nabla{\elboais}_2$ with
{\small
\begin{align}
\nonumber
    &\nabla{\elboais}_1= \int_{}\sum_{\chunk{a}{1}{K}}A(V_{\phi,x}(u_0), \chunk{a}{1}{K}, \chunk{u}{1}{K})\\
    \nonumber&\times \nabla W(V_{\phi,x}(u_0), \chunk{a}{1}{K}, \chunk{u}{1}{K})\densinnovationkp(\chunk{u}{0}{K})\rmd \chunk{u}{0}{K}\eqsp,\\
   \nonumber &\nabla{\elboais}_2= \int_{}\sum_{\chunk{a}{1}{K}}A(V_{\phi,x}(u_0), \chunk{a}{1}{K}, \chunk{u}{1}{K}) W(V_{\phi,x}(u_0), \chunk{a}{1}{K}, \chunk{u}{1}{K})\\ 
   \label{eq:reinforce_ais}
   &\times  \nabla\log A(V_{\phi,x}(u_0), \chunk{a}{1}{K}, \chunk{u}{1}{K}) \densinnovationkp(\chunk{u}{0}{K})\rmd \chunk{u}{0}{K}\eqsp.
\end{align}}
The estimation  of $\nabla{\elboais}_1$ is straightforward. We sample $n$ independent samples $\chunk{u}{0}{K}^{1:n}\sim\densinnovationkp$ and, for $i \in \{1,\dots,n\}$, we set $z_0^i= V_{\phi, x}(u_0^i)$ and then, for $k \in \{1,\dots,K\}$, we sample the A/R variable $a_k^i\sim\Ber\{\alpha_{k, u_k^i}^1(z_{k-1}^i)\}$ and set $z_k^i = \propmap{k, u_k^i}^{a_k^i}(z_{k-1}^i)$, see \Cref{alg:differentiable_ais}. 
Similarly, we then compute
\[
\widehat{\nabla \elboais}_{2,n} =  n^{-1}\sum_{i=1}^n  \nabla W(V_{\phi,x}(u^i_0), \chunk{a^i}{1}{K}, \chunk{u^i}{1}{K}) \eqsp.
\]
The expression ${\nabla \elboais}_2$ is the REINFORCE gradient estimator \cite{williams1992simple} for the A/R probabilities. Indeed, we have to compute  the gradient of the conditional distribution of the A/R variables given $(z_0,\chunk{u}{0}{K})$, and there is no obvious reparametrization for such purpose (see however   \cite{maddison2016concrete} for a possible solution to the problem; this solution was not investigated in this work).
To reduce the variance of the REINFORCE estimator, we rely on control variates, in the spirit of~\cite{mnih2016variational}.
For $i\in\{1,\dots, n\}$, we define
  \[
     \tilde W_{n,i} = \frac{1}{n-1}\sum_{j\neq i} W(V_{\phi,x}(u_0^j), \chunk{a^j}{1}{K}, \chunk{u^j}{1}{K})\eqsp,
  \]
  which is independent of $W(V_{\phi,x}(u_0^i), \chunk{a^i}{1}{K}, \chunk{u^i}{1}{K})$ and $\nabla \log A(V_{\phi,x}(u_0^i), \chunk{a^i}{1}{K}, \chunk{u^i}{1}{K})$ by construction.
  This provides the new unbiased estimator of the gradient using
  \begin{multline}
    \widehat{\nabla \elboais}_{2,n} =  n^{-1}\sum_{i=1}^n\left[W(V_{\phi,x}(u_0^i), \chunk{a^i}{1}{K}, \chunk{u^i}{1}{K}) -  \tilde{W}_{n,i} \right]
      \\
      \times\nabla \log A(V_{\phi,x}(u_0^i), \chunk{a^i}{1}{K}, \chunk{u^i}{1}{K}) \eqsp.
  \end{multline}
  %
 \Cref{alg:differentiable_ais} shows how to compute $W$ and $\log A$.
  \begin{algorithm}[!ht]
    \caption{Annealed Importance Sampling VAE}
    \label{alg:differentiable_ais}
    \begin{algorithmic}
      \STATE {\bfseries Input:} Number of steps $K$, proposal mappings $\{\propmap{k, u}\}_{k\leq K, u\in\msu}$, acceptance functions $\{\alpha_{k, u}\}_{k \leq K, u\in\msu}$, initial distribution $\initdistr$, unnormalized target distribution $p_\theta$, annealing schedule $\{\beta_k\}_{k = 0}^{K}$.
       \STATE {\bfseries Output:} AIS estimator $W$ of $\log p_\theta(x)$, sum $\log A$ of the A/R log probabilities.
        \\\hrulefill
        \STATE Draw $z_0\sim \initdistr(\cdot| x)$;
        \STATE Set $W = 0$;
        \STATE Set $\log A = 0$;
      \FOR{$k=1$ {\bfseries to} $K$}
      \STATE Draw $u_k\sim \densinnovation$;
      \STATE Draw $a_k \sim \Ber(\alpha_{k,u_k}\bigl(z_{k - 1})\bigr)$;
      \IF{$a_k= 1$ (Accept)}
      \STATE Set $z_k = z_{k-1} +\stepsize \nabla \log \unnormdistr_k(z_k)  + \sqrt{2\stepsize}u_k$;
      \ELSE
      \STATE $z_k = z_{k-1}$;
      \ENDIF
      \STATE Compute $\log w_{k}(z_{k - 1}) =$ \\
      \qquad$(\beta_{k} -\beta_{k-1}) \bigl(\log p_\theta(x, z_{k - 1}) - \log\initdistr(z_{k - 1}| x)\bigr)$;
      \STATE Set $W = W + \log w_{k}(z_{k - 1})$;
      \STATE Set $\log A = \log A + \log \alpha^{a_k}_{k, u_k}(z_{k - 1})$;
      \ENDFOR
      \STATE Return $W, \log A$
    \end{algorithmic}
  \end{algorithm}  
\section{Experiments}
\label{sec:experiments}
\subsection{Methods and practical guidelines}
  In what follows, we consider two sets of experiments\footnote{The code to reproduce all of the experiments is available online at \url{https://github.com/premolab/metflow/}.}. In the first one, we aim at illustrating the expressivity and the efficiency of our estimator for VI. In the second, we tackle the problem of learning VAE:
  \begin{enumerate*}[label=(\alph*)]
    \item Classical VAE based on mean-field approximation~\cite{kingma:welling:2014};

    \item Importance-weighted Autoencoder (IWAE, \cite{burda:grosse:2015});


    \item L-MCVAE given by Algorithm~\ref{alg:differentiable_sis_flows};

    \item A-MCVAE given by Algorithm~\ref{alg:differentiable_ais}.
    \end{enumerate*}
  We provide in the following some guidelines on how to tune the step sizes and the annealing schedules in \Cref{alg:differentiable_sis_flows} and \Cref{alg:differentiable_ais}.

  A crucial hyperparameter of our method is the step size $\stepsize$. In principle, it could be learned by including it as an additional inference parameter $\phi$ and by maximizing the ELBO. However, it is then difficult to find a good trade-off between having a high A/R ratio and a large step size $\stepsize$ at the same time. Instead, we suggest adjusting $\stepsize$ by targeting a fixed A/R ratio $\rho$. It has proven effective to use a preconditioned version of \eqref{eq:langevin_eq}, \ie~ 
$Z_{k} = Z_{k - 1} + \stepsize \odot \nabla\log\unnormdistr_k(Z_{k - 1}) + \sqrt{2 \stepsize} \odot U_{k}$ with $\mathbf{\stepsize} \in \rset^p$, where we adapt each component of $\eta$ using the following rule 
  \begin{equation}
    \mathbf{\stepsize}^{(i)} = 0.9 \mathbf{\stepsize}^{(i)} + 0.1 \stepsize_0 / \bigl(\epsilon + \operatorname{std}[ \partial_{z^{(i)}} \log p_\theta(x, z)]\bigr)\eqsp.
  \end{equation}
 Here $\operatorname{std}$ denotes the standard deviation over the batch $x$ of the quantity $\partial_{z^{(i)}} \log p_\theta(x, z)$, and $\epsilon > 0$. The scalar $\stepsize_0$ is a tuning parameter which is adjusted to target the A/R ratio $\rho$. 
  This strategy follows the same heuristics as Adam~\cite{kingma2014adam}.
  In the following $\rho$ is set to 0.8 for A-MCVAE and 0.9 for L-MCVAE (keeping it high for L-MCVAE  ensures that the Langevin dynamics stays ``almost reversible'', thus keeping a low variance SIS estimator).

  An optimal choice of the temperature schedule $\{\beta_k\}_{k = 0}^K$ for SIS and AIS is a difficult problem.   We have focused in our experiments on three different settings. First, we consider the temperature schedule fixed and regularly spaced between 0 and 1. Following~\cite{grosse:Ghahramani:adams:2015}, the second option is the sigmoidal tempering scheme where $\beta_k = (\tilde \beta_k - \tilde \beta_1)/(\tilde \beta_K-\tilde \beta_1)$ with, $\tilde \beta_k = \sigma\bigl(\delta(2k / K -1)\bigr)$,  $\sigma$ is the sigmoid function and $\delta >0$ is a parameter that we optimize during the training phase. The last schedule consists in learning  the temperatures $\{\beta_k\}_{k = 0}^K$ directly as additional inference parameters $\phi$.

\subsection{Toy 2D example and Probabilistic Principal Component Analysis}
\label{subsec:toy_ex_ppca}
In the following two examples, we fix the parameters $\theta$ of the likelihood model and apply \Cref{alg:differentiable_sis_flows} and \Cref{alg:differentiable_ais} to perform VI to sample from  $z \mapsto p_{\theta}(z  | x)$. 
  Consider first a toy hierarchical example where we generate some \iid~data $x=(x_i)_{i=1}^N \in \rset^N$  from the \iid~latent variables $z =(z_i)_{i=1}^N \in \rset^{2N}$ as follows for $\xi >0$ 
 $x_i| z_i\sim \Normal(\xi \cdot (\|z_i\|^2 + \zeta), \sigma^2)= p_\theta(x_i| z_i)$.
 We consider the variational approximation as $q_\phi(z| x) = \Normal\bigl(z; \mu_\phi(x),\sigma_\phi(x)^2\Id\bigr)$, where $\mu_\phi(x), \sigma_\phi(x)\in\rset^d$ are the outputs of a fully connected neural network with parameters $\phi$.
 We compare these algorithms to VI using Real-valued Non-Volume Preserving transformation (RealNVP, \citet{dinh2016density}).
  \begin{figure}
    \centering
    \renewcommand{\arraystretch}{0.1}
    \begin{tabular}{ccc}
      \hspace{-10pt}\includegraphics[width=0.3\linewidth]{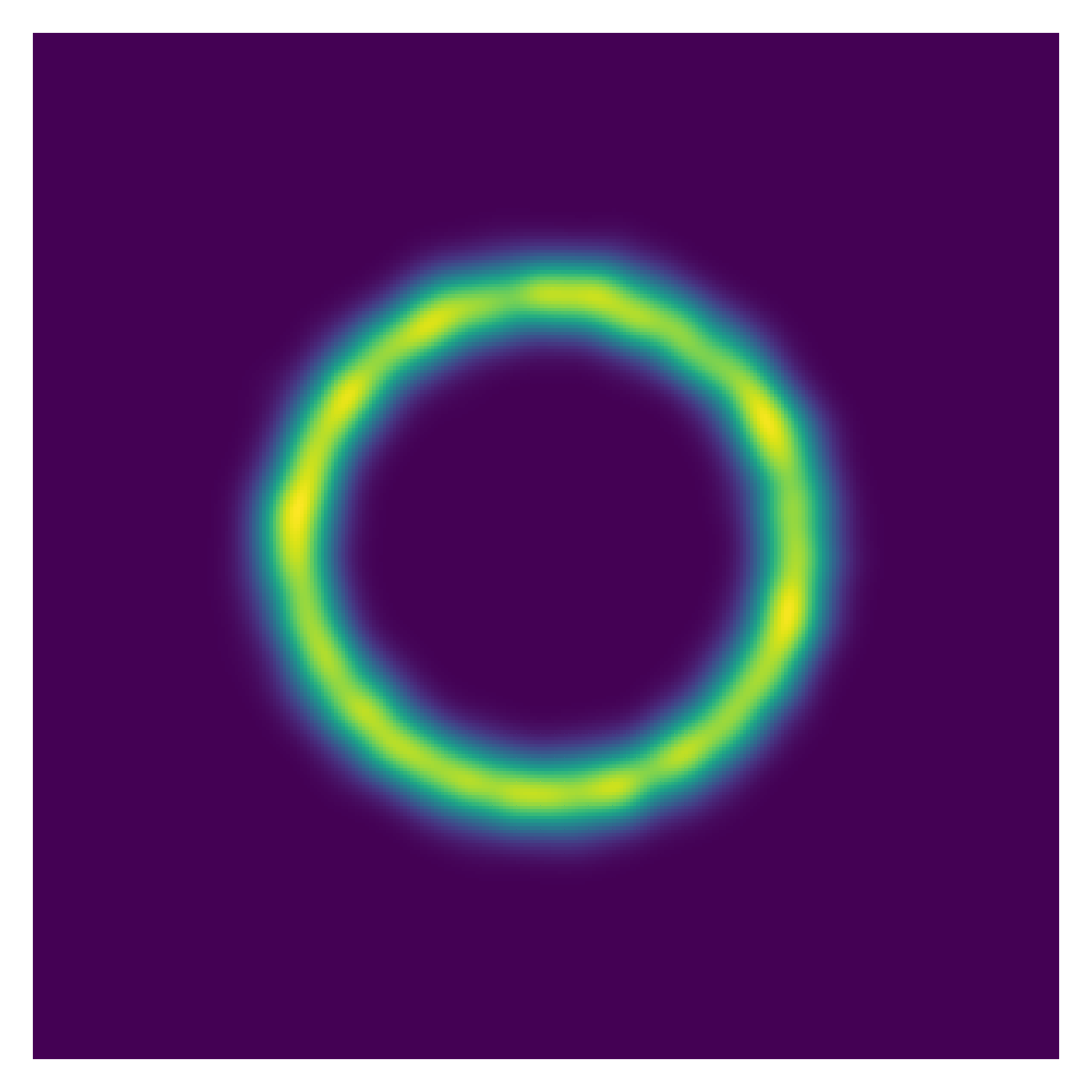} & \hspace{-20pt}\includegraphics[width=0.3\linewidth]{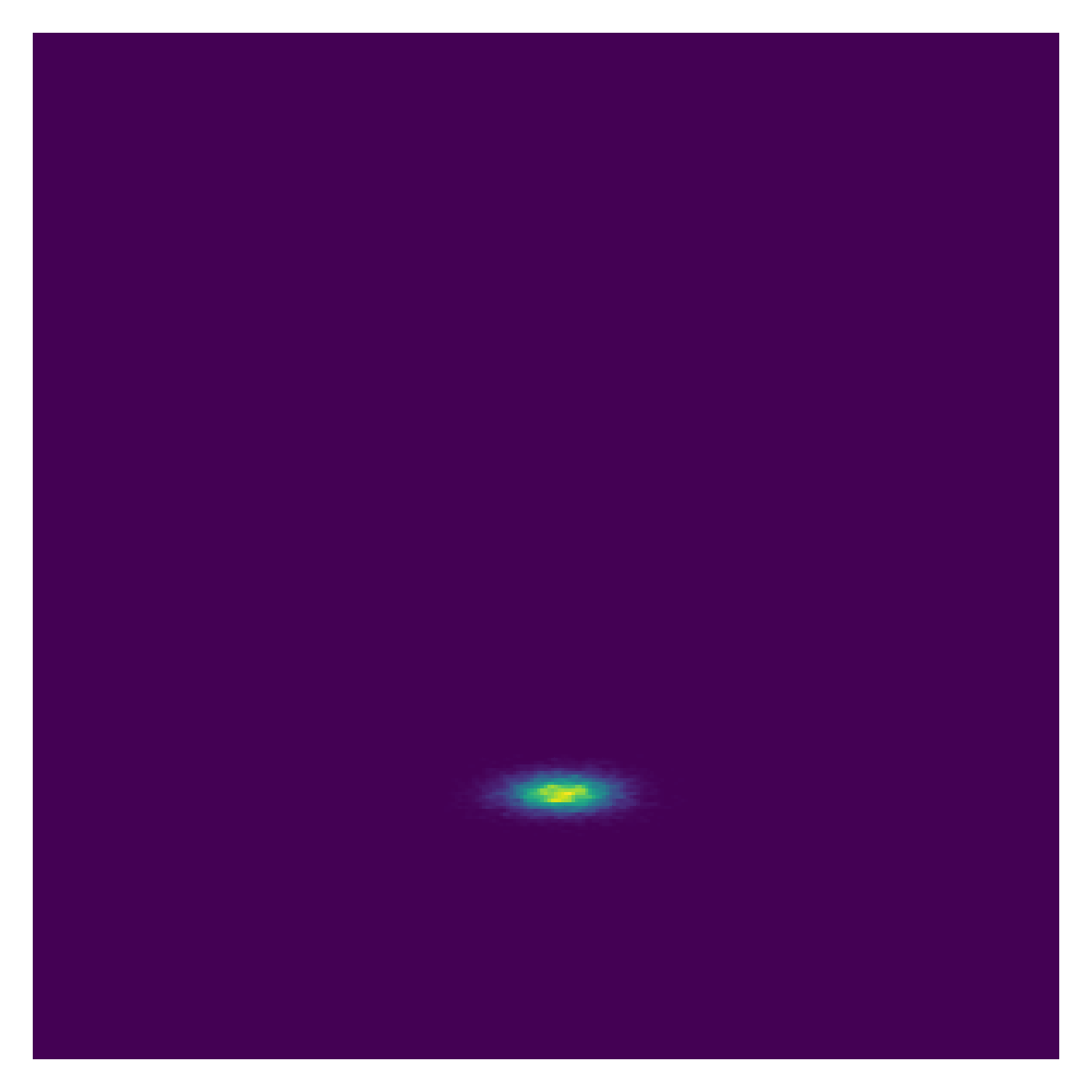} & \hspace{-20pt}\includegraphics[width=0.3\linewidth]{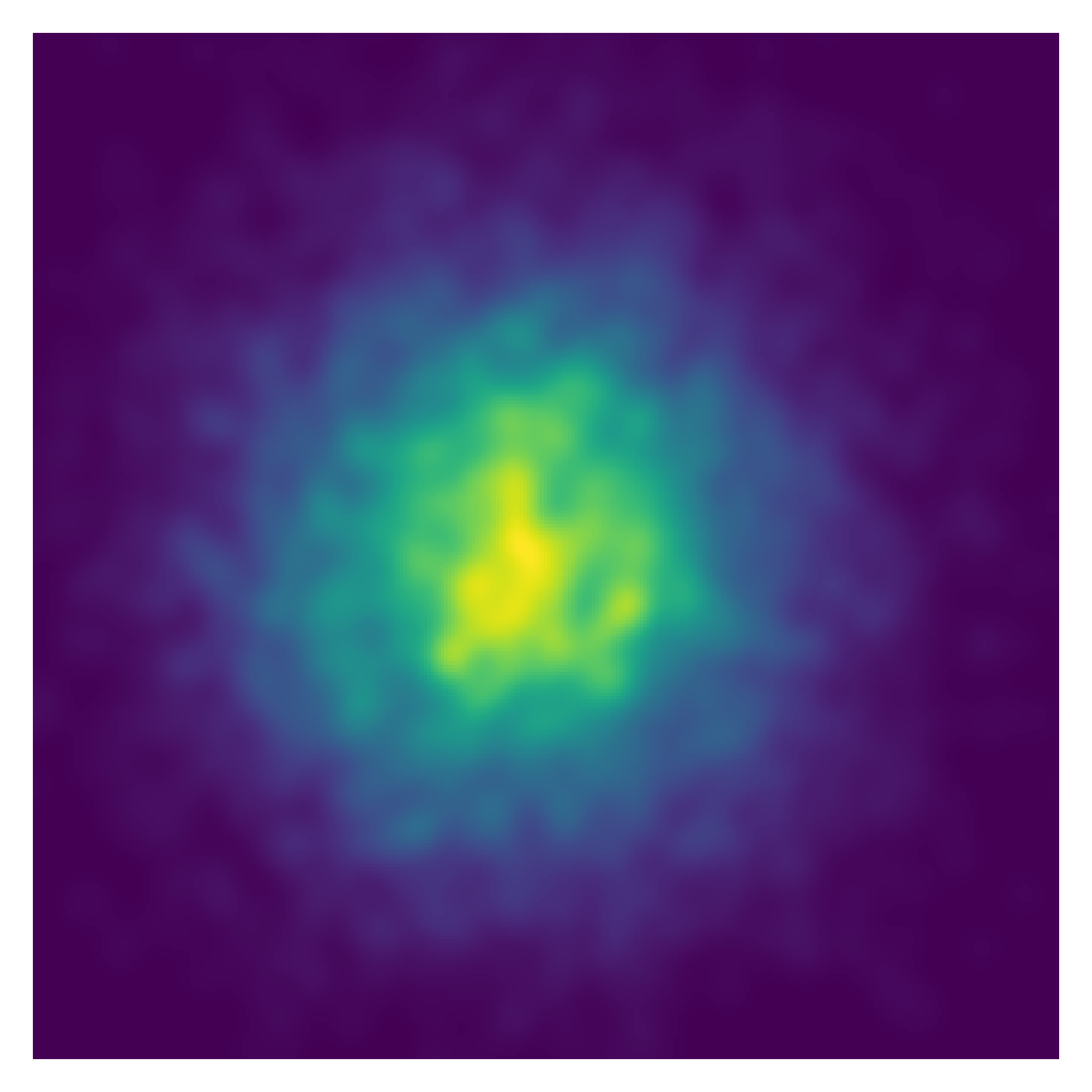} \\
      \hspace{-10pt}\includegraphics[width=0.3\linewidth]{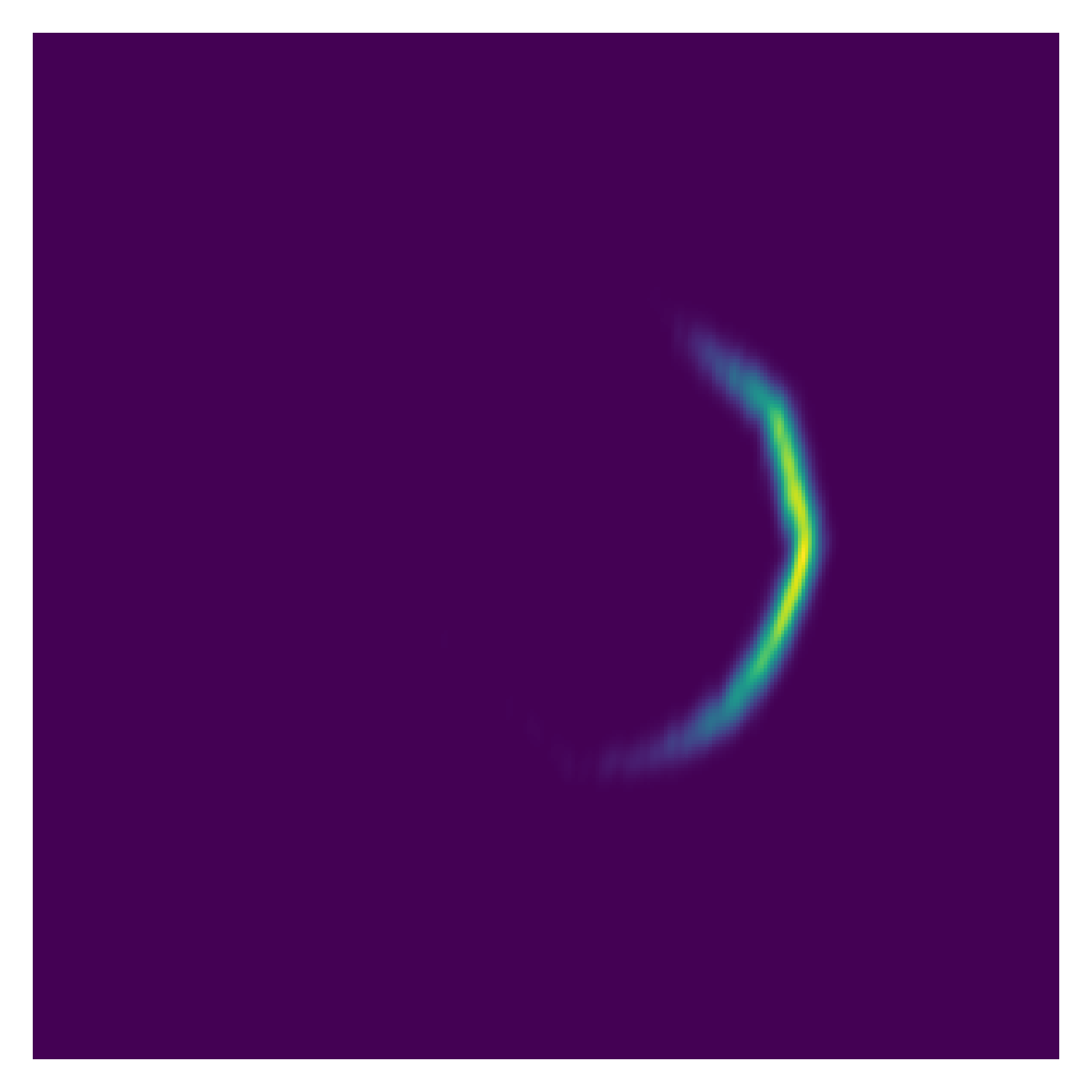} & \hspace{-20pt}\includegraphics[width=0.3\linewidth]{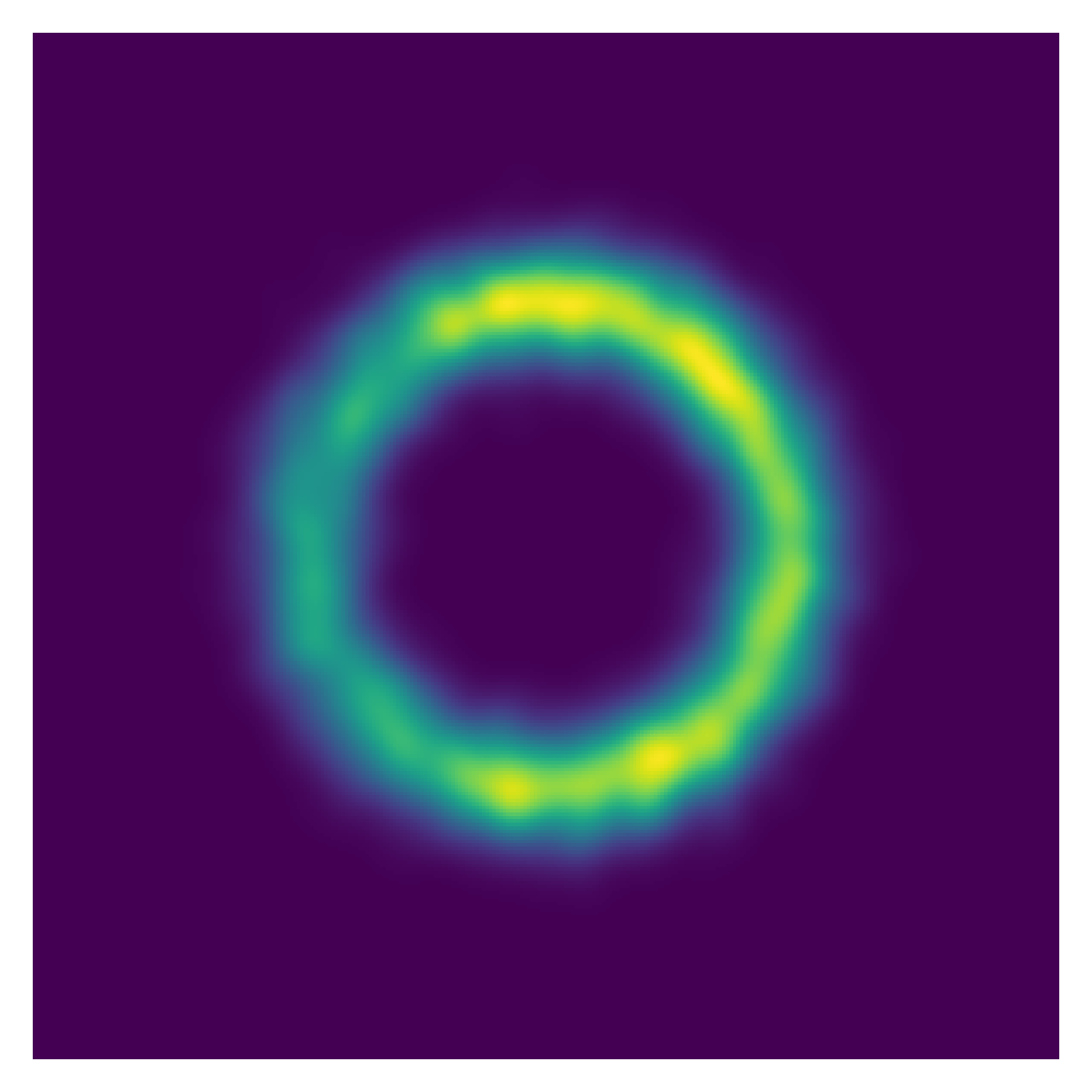} & \hspace{-20pt}\includegraphics[width=0.3\linewidth]{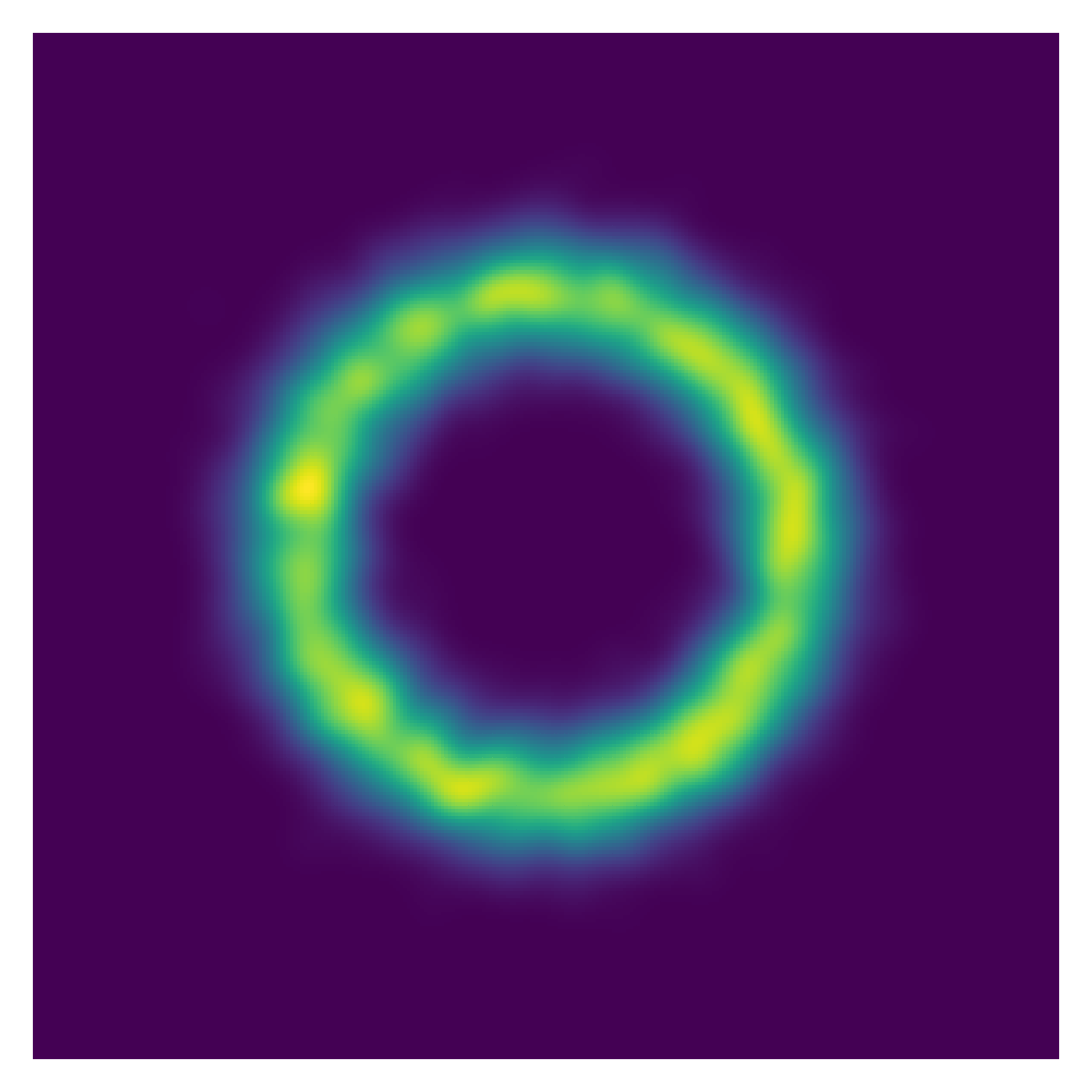}
    \end{tabular}
    \caption{Visualization of the posterior approximation given after optimization of different bounds for toy generation process. Top row, from left to right: True posterior, VAE posterior, IWAE posterior. Bottom row, from left to right: VI with RealNVP posterior, A-MCVAE posterior, L-MCVAE posterior.}
  \label{fig:toy_example_posterior}
  \end{figure}
  
  \Cref{fig:toy_example_posterior} displays the VI posterior approximations corresponding to the different schemes for a given observation $x$. It can be observed that MCVAE benefit from flexible variational distributions compared to other classical schemes, which mostly fail to recover the true posterior. Additional results on the estimation of the parameters $\xi, \zeta$, given in the supplementary material, further support our claims; see \Cref{spsec:exps_toy}.
  
  We now illustrate the performance of MCVAE on a probabilistic principal component analysis problem applied to MNIST~\cite{salakhutdinov2008quantitative}, as we can access in this case the exact likelihood and its gradient. We follow the set-up of~\citep[Section 6.1]{ruiz:titsias:doucet:2020}.
  \begin{figure}[!ht]
     \centering
   \hspace{8pt}  \includegraphics[width=\linewidth]{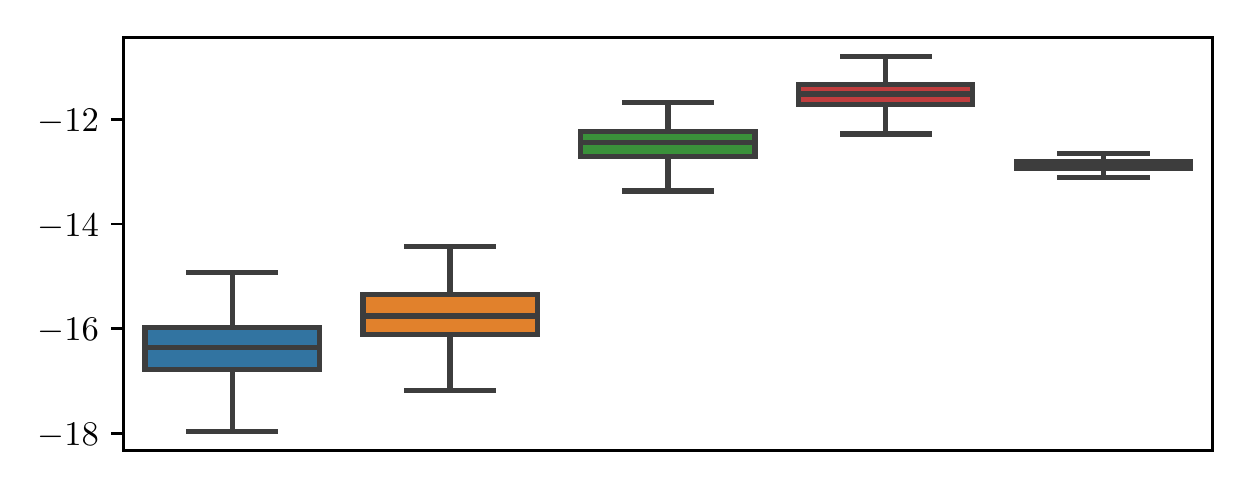}
     \includegraphics[width=\linewidth]{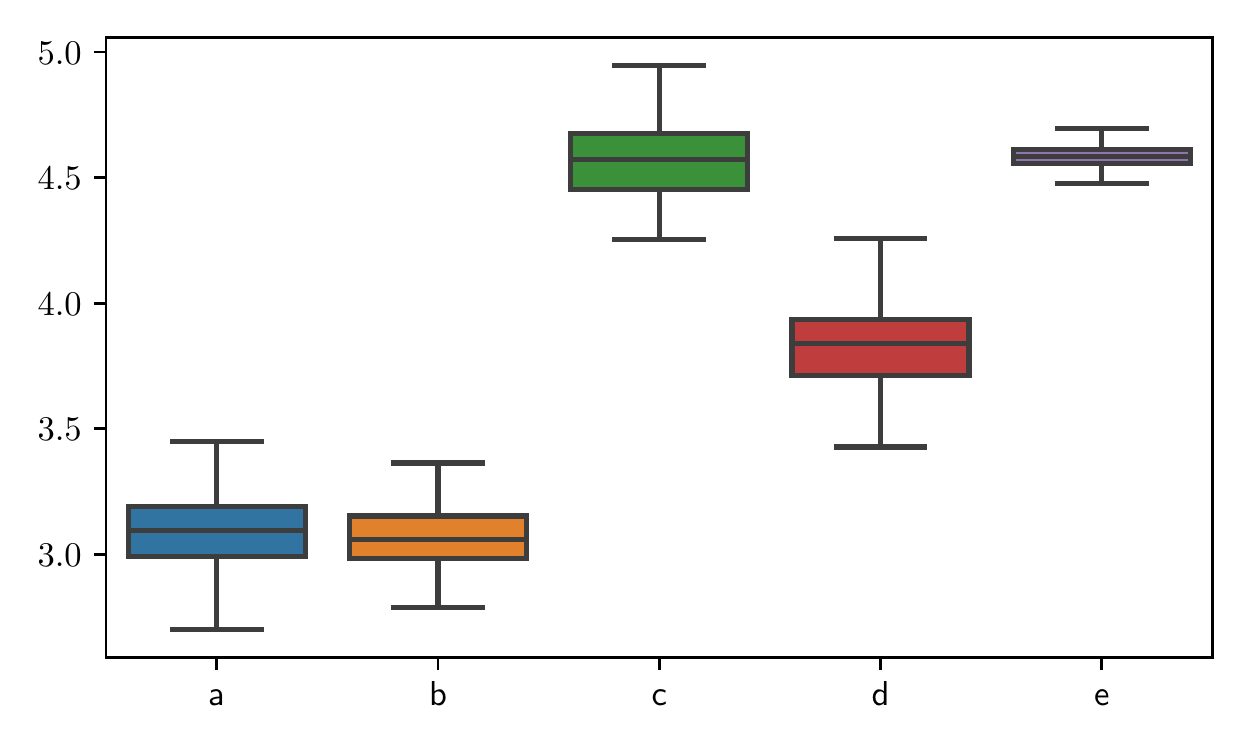}
     \caption{Representation of the different estimators (top) and their gradient (bottom) of the true log likelihood. From left to right, a/ L-MCVAE, $K=5$, b/ L-MCVAE, $K=10$, c/ A-MCVAE, $K=5$, d/ A-MCVAE, $K=10$, e/ A-MCVAE, $K=5$ with control variates.}
     \label{fig:boxplot_ppca}
 \end{figure}
  We consider here a batch of size $N=100$ for the model $p_\theta(x,z) = \Normal(z;0, \Id_d)\Normal(x; \theta_0 + \theta_1 z, \sigma^2\Id_p)
  $, with $d=100$ and $p=784$. We fix arbitrarily $\theta_0, \theta_1$, and fit an amortized variational distribution $q_\phi(z | x)$ by maximizing the IWAE bound \wrt\ $\phi$ with $K= 100$ importance samples for a large number of epochs. The distribution $q_\phi(z | x)= \Normal(z; \mu_\phi(x), \operatorname{diag}(\sigma_\phi^2(x)))$ is a mean-field Gaussian distribution where $\mu_\phi, \sigma_{\phi}$ are linear functions of the observation $x$.
 
  We compare the Langevin SIS estimator (L-MCVAE) of the log evidence $\log p_\theta(x)$ with Langevin auxiliary kernels as described in \Cref{subsec:LangevinELBO}, and the Langevin AIS estimate (A-MCVAE).
  Moreover, we also compare the gradients of these quantities \wrt~the parameters $\theta_0, \theta_1$.
  
  \Cref{fig:boxplot_ppca} summarises the results with boxplots computed over $200$ independent samples of each estimator. 
  The quantity reported on the first boxplot corresponds to the Monte Carlo samples of $\log \hat{p}_\theta(x) - \log p_\theta(x)$. One the one hand, we note that the SIS estimator has larger variance than AIS, and that the latter achieves a better ELBO. Moreover, in both cases, increasing the number of steps $K$ tightens the bound. On the other hand, the estimator of the gradient of AIS is noisier than that of SIS, even though variance reduction techniques allows us to recover a similar variance.
  We also present in the supplementary material the Langevin SIS estimator using auxiliary backward kernels learnt with neural networks (as done in previous contributions); see \Cref{spsec:exps_ppca}.
  The auxiliary neural backward kernels are set as $l(z, z') = \Normal(z'; \mu_\psi(z), \operatorname{diag}(\sigma^2_\psi(z)))$,  $\mu_\psi, \sigma_\psi\in\rset^d$, where the parameters $\psi$ are learnt through the SIS ELBO, similarly to \cite{huang:tan:lacoste:courville:2018}.
  The variance of the associated estimator and their gradients are larger than that of SIS using the approximate reversals as backward kernels; i.e. $\ell_{k-1}=m_k$.
  
  \begin{figure}[ht!]
    \centering
    \includegraphics[width = .85\linewidth]{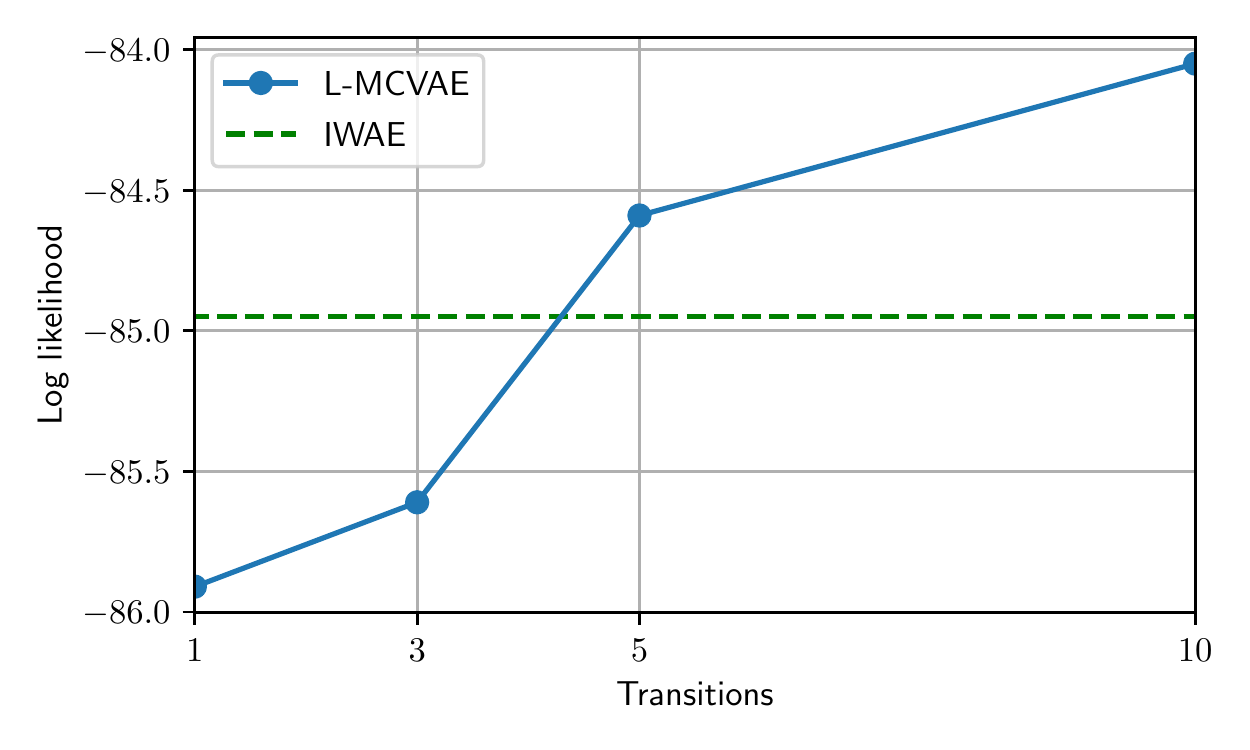}
    \caption{Log-likelihood of L-MCVAE depending on the number of Langevin steps $K$. Increasing $K$ improves performance, however at the expense of the computational complexity.}
    \label{fig:nll_k_lmcvae}
  \end{figure}
  \begin{figure}
    \centering
    \includegraphics[width = .85\linewidth]{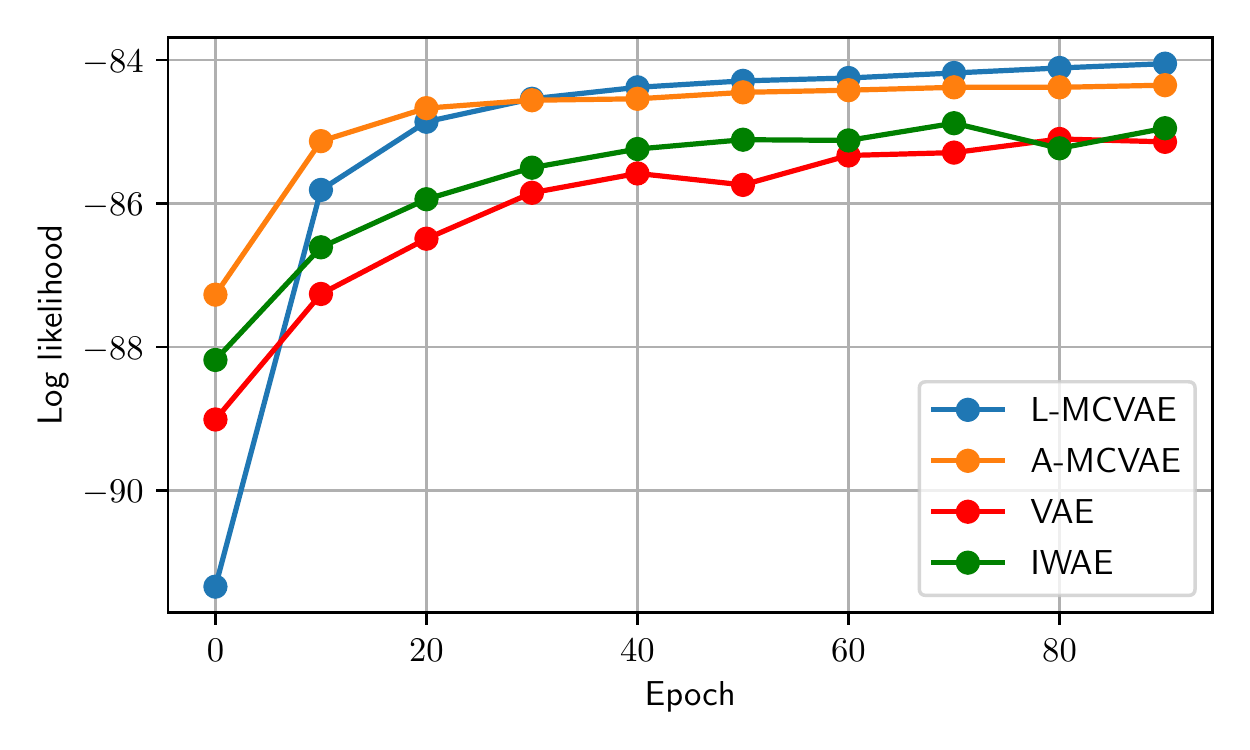}      \caption{Evolution of the held-out loglikelihood during training for A-MCVAE, L-MCVAE, IWAE and VAE on MNIST.
    }
  \label{fig:nll_training}
  \end{figure}
\subsection{Numerical results for image datasets}
  \begin{table*}[ht!]
    \centering
    \caption{Results of the different models on MNIST. A more detailed version of this table is included in the supplementary material.}
    \label{tab:MNIST_1}
    \begin{tabular}{l|l|l|l||l|l|l|}
      & \multicolumn{3}{c||}{\emph{negative ELBO estimate}} & \multicolumn{3}{c|}{\emph{NLL estimate}} \\ \hline
      \multicolumn{1}{|c|}{\emph{number of epochs}}  & \multicolumn{1}{r|}{\textit{10}} & \multicolumn{1}{r|}{\textit{30}} & \multicolumn{1}{r||}{\textit{100}} & \multicolumn{1}{r|}{\textit{10}} & \multicolumn{1}{r|}{\textit{30}} & \multicolumn{1}{r|}{\textit{100}}
      \\ \hline
      \multicolumn{1}{|l|}{VAE}               &   95.26 $\pm$ 0.49        &  91.58 $\pm$ 0.27         & 89.70 $\pm$ 0.19           &  89.83 $\pm$ 0.59         & 86.86 $\pm$ 0.26          & 85.22 $\pm$ 0.07          \\ \hline
      \multicolumn{1}{|l|}{IWAE, $K=10$}      & 91.42 $\pm$ 0.21 & 88.56 $\pm$ 0.07 & 87.16 $\pm$ 0.19         & 88.54 $\pm$ 0.27 &  86.07 $\pm$ 0.1 &  84.82 $\pm$ 0.1 \\
      \multicolumn{1}{|l|}{IWAE, $K=50$}  &    90.34 $\pm$ 0.27 &  87.5 $\pm$ 0.16 & 86.05 $\pm$ 0.11 &  89.4 $\pm$ 0.25 & 86.54 $\pm$ 0.15 &  85.05 $\pm$ 0.1 \\ \hline
      
     \multicolumn{1}{|l|}{L-MCVAE, $K=5$}    & 96.62 $\pm$ 3.24 & 88.58 $\pm$ 0.75 & 87.51 $\pm$ 0.41 & 90.59 $\pm$ 2.01 & 85.68 $\pm$ 0.49 & 84.92 $\pm$ 0.24 \\
     
      \multicolumn{1}{|l|}{L-MCVAE, $K=10$}   & 96.78 $\pm$ 1.06 & 87.99 $\pm$ 0.71 &  86.8 $\pm$ 0.66 & 91.33 $\pm$ 0.61 & 85.47 $\pm$ 0.46 & \textbf{84.58} $\pm$ 0.39 \\ \hline
      
      \multicolumn{1}{|l|}{A-MCVAE, $K=3$}   & 96.21 $\pm$ 3.43 & 88.64 $\pm$ 0.78 & 87.63 $\pm$ 0.42 & 90.42 $\pm$ 2.34 & 85.77 $\pm$ 0.65 & 85.02 $\pm$ 0.37 \\ 
      \multicolumn{1}{|l|}{A-MCVAE, $K=5$}    & 95.55 $\pm$ 2.96 & 87.99 $\pm$ 0.57 & 87.03 $\pm$ 0.27 & 90.39 $\pm$ 2.21 &  85.6 $\pm$ 0.67 & \textbf{84.84} $\pm$ 0.38 \\ \hline
      \multicolumn{1}{|l|}{VAE with RealNVP}    &  95.23 $\pm$ 0.33 &  91.69 $\pm$ 0.15 & 89.62 $\pm$ 0.17 & 89.98 $\pm$ 0.24 & 86.88 $\pm$ 0.05 &  85.23 $\pm$ 0.18 \\  \hline
      \end{tabular}
      \end{table*}
      \begin{table*}[ht!]
    \centering
    \caption{Results of the different models on CelebA. A more detailed version of this table is included in the supplementary material. 11400 must be added to all scores in this table.}
    \label{tab:CelebA_1}
    \begin{tabular}{l|l|l|l||l|l|l|}
      & \multicolumn{3}{c||}{\emph{negative ELBO - }11400+} & \multicolumn{3}{c|}{\emph{NLL - }11400+} \\ \hline
      \multicolumn{1}{|c|}{\emph{number of epochs}}  & \multicolumn{1}{r|}{\textit{10}} & \multicolumn{1}{r|}{\textit{30}} & \multicolumn{1}{r||}{\textit{100}} & \multicolumn{1}{r|}{\textit{10}} & \multicolumn{1}{r|}{\textit{30}} & \multicolumn{1}{r|}{\textit{100}}
      \\ \hline
      \multicolumn{1}{|l|}{VAE}               & 23.78 $\pm$ 1.95 &  17.99 $\pm$ 0.4 & 14.72 $\pm$ 0.16 &  17.35 $\pm$ 1.7 & 12.68 $\pm$ 0.62 & 10.11 $\pm$ 0.32 \\ \hline
      \multicolumn{1}{|l|}{IWAE, $K=10$}      & 20.59 $\pm$ 0.71 & 15.45 $\pm$ 0.52 &   12.2 $\pm$ 0.3 &  18.25 $\pm$ 0.6 & 13.18 $\pm$ 0.42 & 10.14 $\pm$ 0.31 \\
      \multicolumn{1}{|l|}{IWAE, $K=50$}      & 19.05 $\pm$ 0.39 &  13.59 $\pm$ 0.5 & 10.48 $\pm$ 0.89 & 19.08 $\pm$ 0.42 & 13.17 $\pm$ 0.54 & 10.12 $\pm$ 0.86 \\ \hline
      \multicolumn{1}{|l|}{L-MCVAE, $K=5$}    & 21.61 $\pm$ 1.48 & 12.72 $\pm$ 0.43 &  11.6 $\pm$ 0.37 & 16.42 $\pm$ 1.47 & 9.62 $\pm$ 0.47 &  8.72 $\pm$ 0.4 \\
       \multicolumn{1}{|l|}{L-MCVAE, $K=10$}   &  20.7 $\pm$ 1.15 & 11.81 $\pm$ 0.34 &  10.6 $\pm$ 0.23 &  17.0 $\pm$ 1.87 & 9.29 $\pm$ 0.73 & \textbf{8.24} $\pm$ 0.52 \\ \hline
      \multicolumn{1}{|l|}{A-MCVAE, $K=3$}   &  21.59 $\pm$ 1.5 & 13.94 $\pm$ 0.42 &  12.84 $\pm$ 0.3 & 16.64 $\pm$ 1.37 & 10.98 $\pm$ 0.48 &  9.95 $\pm$ 0.3 \\ 
      \multicolumn{1}{|l|}{A-MCVAE, $K=5$}    & 20.95 $\pm$ 1.18 & 12.42 $\pm$ 0.42 & 11.13 $\pm$ 0.37 & 17.42 $\pm$ 1.49 & 9.97 $\pm$ 0.59 & \textbf{8.82} $\pm$ 0.57 \\ \hline
      \multicolumn{1}{|l|}{VAE with RealNVP}    &  15.12 $\pm$ 0.48 &  13.63 $\pm$ 0.27 & 12.58 $\pm$ 0.61 & 10.42 $\pm$ 0.33 & 9.04 $\pm$ 0.26 &  8.98 $\pm$ 0.2 \\    \hline
    \end{tabular}
  \end{table*}
   \begin{table*}[ht!]
    \centering
  \caption{Results of the different models on CIFAR. A more detailed version of this table is included in the supplementary material. 2800 must be added to all scores in this table.}
    \label{tab:CIFAR_1}
\begin{tabular}{l|l|l|l||l|l|l|}
      & \multicolumn{3}{c||}{\emph{negative ELBO - }2800+} & \multicolumn{3}{c|}{\emph{NLL - }2800+} \\ \hline
      \multicolumn{1}{|c|}{\emph{number of epochs}}  & \multicolumn{1}{r|}{\textit{10}} & \multicolumn{1}{r|}{\textit{30}} & \multicolumn{1}{r||}{\textit{100}} & \multicolumn{1}{r|}{\textit{10}} & \multicolumn{1}{r|}{\textit{30}} & \multicolumn{1}{r|}{\textit{100}}\\
\hline
                          \multicolumn{1}{|l|}{          VAE} & 69.57 $\pm$ 0.08 & 69.55 $\pm$ 0.51 & 68.84 $\pm$ 0.06 & 68.51 $\pm$ 0.07 & 68.41 $\pm$ 0.33 &  67.9 $\pm$ 0.03 \\
                             \multicolumn{1}{|l|}{IWAE, K= 10} & 69.82 $\pm$ 0.03 & 69.35 $\pm$ 0.03 & 69.36 $\pm$ 0.36 & 68.56 $\pm$ 0.03 &  68.0 $\pm$ 0.03 &  68.02 $\pm$ 0.4 \\
                             \multicolumn{1}{|l|}{IWAE, K= 50} & 69.94 $\pm$ 0.08 & 69.55 $\pm$ 0.04 & 69.43 $\pm$ 0.03 & 69.15 $\pm$ 0.15 & 68.37 $\pm$ 0.18 & 67.93 $\pm$ 0.02 \\
                             \hline
           \multicolumn{1}{|l|}{L-MCVAE, K= 5} & 70.62 $\pm$ 0.41 & 68.55 $\pm$ 0.18 &  68.09 $\pm$ 0.1 & 69.15 $\pm$ 0.38 & 67.73 $\pm$ 0.07 &  \textbf{67.5} $\pm$ 0.07 \\
              \multicolumn{1}{|l|}{L-MCVAE, K= 10} & 70.99 $\pm$ 0.59 & 68.36 $\pm$ 0.04 &  68.03 $\pm$ 0.0 &  69.8 $\pm$ 0.67 & 67.76 $\pm$ 0.04 & \textbf{67.51} $\pm$ 0.03 \\
          \hline
       \multicolumn{1}{|l|}{A-MCVAE, K= 3} & 69.97 $\pm$ 0.99 & 68.48 $\pm$ 0.29 & 68.18 $\pm$ 0.16 & 69.26 $\pm$ 0.76 & 67.77 $\pm$ 0.18 &  67.55 $\pm$ 0.1 \\
\multicolumn{1}{|l|}{A-MCVAE, K= 5} &  70.1 $\pm$ 0.89 &  68.28 $\pm$ 0.2 & 68.01 $\pm$ 0.08 & 69.23 $\pm$ 0.75 & 67.71 $\pm$ 0.15 &  \textbf{67.5} $\pm$ 0.07 \\
\hline
 \multicolumn{1}{|l|}{VAE with RealNVP}
 &  70.01 $\pm$ 0.12 &  69.51 $\pm$ 0.07 & 69.19 $\pm$ 0.13 & 68.73 $\pm$ 0.05 & 68.35 $\pm$ 0.05 &  68.05 $\pm$ 0.02 \\
\bottomrule
\end{tabular}
  \end{table*}
  Following~\cite{wu:burda:grosse:2016}, we propose to evaluate our models using AIS (not to be confused with the proposed AIS-based VI approach) to get an estimation of the negative log-likelihood.
  The base distribution is the distribution output by the encoder, and we perform $K$ steps of annealing to compute the estimator of the likelihood, as given by~\eqref{eq:SiS-estimator-Z}.
  In practice, we use $K = 5$ HMC steps with 3 leapfrogs for evaluating our models.
  
  We evaluate our models on three different datasets: MNIST, CIFAR-10 and CelebA.
  All the models we compare share the same architecture: the inference network $q_\phi$ is given by a convolutional network with 8 convolutional layers and one linear layer, which outputs the parameters $\mu_\phi(x), \sigma_\phi(x)\in\rset^d$ of a factorized Gaussian distribution, while the generative model $p_\theta(\cdot | z)$ is given by another convolutional network $\pi_\theta$, where we use nearest neighbor upsamplings. This outputs the parameters for the factorized Bernoulli distribution (for MNIST dataset), that is
  $$
    p_\theta(x | z) = \prod_{i=1}^N \Ber\Bigl(x^{(i)} | \bigl(\pi_\theta(z)\bigr)^{(i)}\Bigr)  
 $$
 and similarly the mean of the Gaussian distributions for colored datasets (CIFAR-10, Celeba).
  We compare A-MCVAE, L-MCVAE, IWAE, and VAE with different settings. 
  All the models are implemented using PyTorch~\cite{paszke2019pytorch} and optimized using the Adam optimizer~\cite{kingma2014adam} for 100 epochs each. The training process is using PyTorch Lightning toolkit~\cite{falcon2019pytorch}.
  
  First, consider dynamically binarized MNIST dataset~\cite{salakhutdinov2008quantitative}. In this case, the latent dimension is set to $d = 64$.
  We present in \Cref{tab:MNIST_1} the results of the different models at different stages of the optimization.
  Moreover, we show on \Cref{fig:nll_k_lmcvae} the performance of L-MCVAE for different values of $K$ compared to IWAE baseline. In particular, we see that increasing $K$ increases the performance of our VAE, however at the expense of an increase in computational cost. 
  We also display on \Cref{fig:nll_training} the evolution of the held-out loglikelihood for various objectives during training. Adding Langevin transitions appears to help convergence of the models.
  
  Second, we compare similarly the different models on CelebA and CIFAR, see \Cref{tab:CelebA_1} and \Cref{tab:CIFAR_1}. In this case, the latent dimension is chosen to be $d= 128$. 
  Increasing the number of MCMC steps seems again to improve both the ELBO and the final loglikelihood estimate. In each case, all models are run with 5 different seeds to compute the presented empirical standard deviation.

\section{Discussion}
We have shown in this article how one can leverage state-of-the-art Monte Carlo estimators of the evidence to develop novel competitive VAEs by developing novel gradient estimates of the corresponding ELBOs. 

For a given computational complexity, AIS based on MALA provides ELBO estimates which are typically tighter than SIS estimates based on ULA.
However, the variance of the gradient estimates of the AIS-based ELBO (A-MCVAE) is also significantly larger than for the SIS-based ELBO (L-MCVAE) as it has to rely on REINFORCE gradient estimates. While control variates can be considered to reduce the variance, this comes at a significant increase in computational cost. 

Empirically, L-MCVAE should thus be favoured as it provides both a tighter ELBO than standard techniques and low variance gradient estimates.

\section*{Acknowledgements}
The work was partly supported by ANR-19-CHIA-0002-01 ``SCAI'' and EPSRC CoSInES grant EP/R034710/1. It was partly carried out under the framework of HSE University Basic Research Program. The development of a software system for the experimental study of VAEs and its application to computer vision problems (Section 4) was supported by the Russian Science Foundation grant 20-71-10135. Part of this research has been carried out under the auspice of the Lagrange Center for Mathematics and Computing.

\clearpage
\newpage
\bibliographystyle{icml2021}
\bibliography{bibliography}


\icmltitlerunning{}

\setcounter{equation}{0}
\setcounter{figure}{0}
\setcounter{table}{0}
\setcounter{page}{1}
 \renewcommand{\theequation}{S\arabic{equation}}
 \renewcommand{\thefigure}{S\arabic{figure}}
 \renewcommand{\thetheorem}{S\arabic{theorem}}
 \renewcommand{\thelemma}{S\arabic{lemma}}
 \renewcommand{\theproposition}{S\arabic{proposition}}

\appendix
\onecolumn
\icmltitle{{Monte Carlo VAE \\
\normalsize SUPPLEMENTARY DOCUMENT}}

\section{Notations and definitions}
\label{sec:notations-supp}
  Let $(\Xset,\Xsigma)$ be a measurable space. A
  \emph{Markov kernel} $N$ on $\Xset \times \Xsigma$ is a mapping $N: \Xset \times \Xsigma \to \ccint{0,1}$ satisfying the following conditions:
\begin{enumerate}[label={(\roman*)}]
\item \label{item:def-kernel-measure} for every $x \in \Xset$, the mapping $N(x,\cdot): A \mapsto N(x, A)$ is a probability of  on $\Xsigma$,
\item \label{item:def-kernel-measurable} for every $A \in \Xsigma$, the mapping $N(\cdot,A): x\mapsto  N(x,A)$ is a measurable function from $(\Xset,\Xsigma)$ to $(\ccint{0,1},\mathcal{B}(\ccint{0,1})$, where $\mathcal{B}(\ccint{0,1})$ denotes the borelian sets of $\ccint{0,1}$.
\end{enumerate}
Let $\lambda$ be a positive $\sigma$-finite measure on $(\Xset,\Xsigma)$ and
  $n: \Xset \times \Xset \to \rset_+$ be a nonnegative function, measurable with
  respect to the product $\sigma$-field $\Xsigma \otimes \Xsigma$. Then, the
  application $N$ defined on $\Xset\times\Xsigma$ by
\[
N(x,A)\ =\int_{A} n(x, y) \lambda(\rmd y) \eqsp,
\]
is a kernel.  The function $n$ is called the density of the kernel $N$ \wrt\ the
measure $\lambda$. The kernel $N$ is Markovian if and only if $\int_{\Xset}
n(x,y)\lambda(\rmd y)=1$ for all $x\in\Xset$.

Let $N$ be a kernel on $\Xset \times \Xsigma$ and $f$ be a nonnegative function.  A function $Nf: \Xset \to
\rset_+$ is defined by setting, for  $x \in \Xset$,
\[
Nf(x) = \int_{\Xset} N(x, \rmd y)f(y) \eqsp.
\]
Let $\mu$ be a probability on $(\Xset,\Xsigma)$. For $A\in
\Xsigma$, define
\[
\mu N(A)=\int_{\Xset} \mu(\rmd x)\ N(x,\ A) \eqsp.
\]
If $N$ is Markovian, then $\mu N$ is a probability on $(\Xset,\Xsigma)$.

\section{Experiences}
\subsection{Toy example}
\label{spsec:exps_toy}
We first describe additional experiments on the toy dataset introduced in \Cref{subsec:toy_ex_ppca}.

Recall that we generate some \iid~data $x=(x_i)_{i=1}^N \in \rset^N$  from the \iid~latent variables $z =(z_i)_{i=1}^N \in \rset^{2N}$ as follows for $\eta >0$: $z_i \sim \Normal(0; \Id)$ and 
 $x_i\mid z_i\sim \Normal(\eta \cdot (\|z_i\| + \zeta), \sigma^2)= p_\theta(x_i\mid z_i)$.


 This example, presented for $z\in\rset^2$, easily extends to the case where $z$ lies in $\rset^d$, with $d$ increasing from 2 to 300.
We tackle here the problem at estimating the parameter $\theta = (\eta, \zeta)$ when $d$ varies.

We show in \Cref{fig:toy_example_estim} the error $\|\hat \theta - \theta\|^2$ for the different methods.
\begin{figure}[!ht]
     \centering
   \includegraphics[width=.49\linewidth]{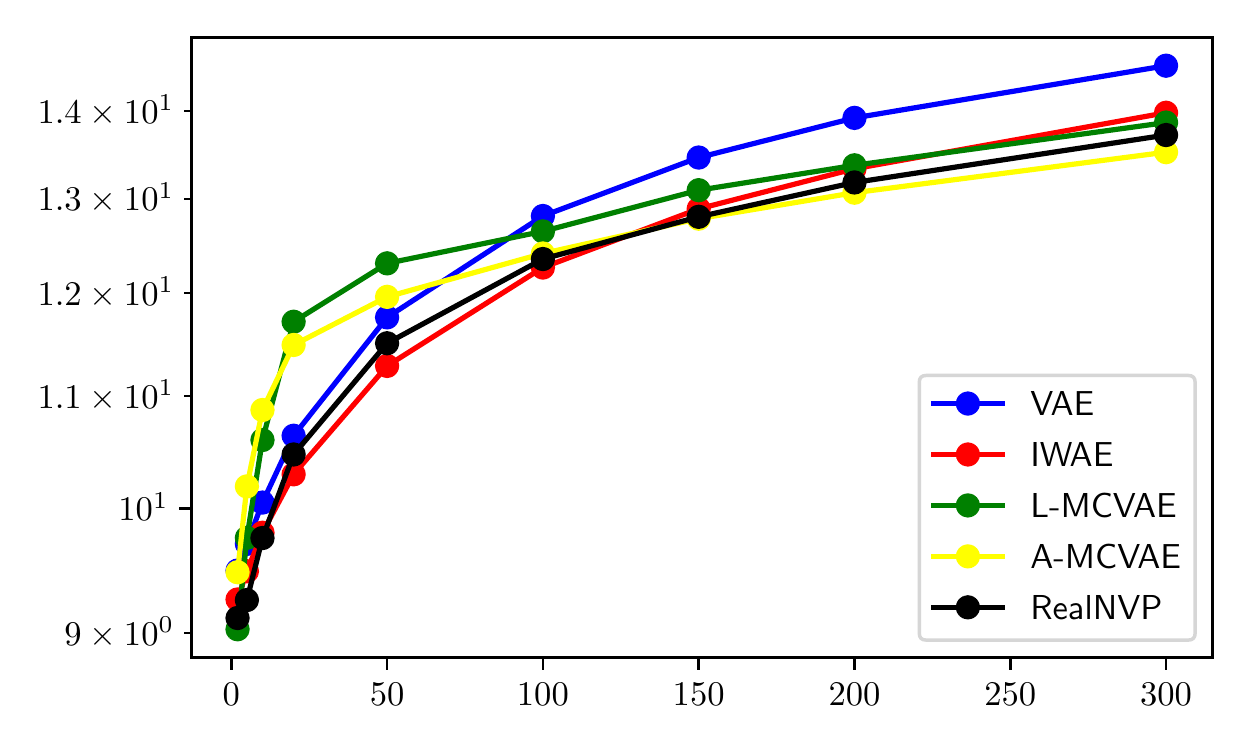}
     \caption{Squared error for parameter's estimates, obtained using different models.}
     \label{fig:toy_example_estim}
 \end{figure}
The increased flexibility of the posterior proves more effective for estimating the true parameters of the generative model.

\subsection{Probabilistic Principal Component Analysis}
\label{spsec:exps_ppca}
We detail the impact of the learnable reverse kernels on the variance of the estimator and looseness of the ELBO. In our experiments, reverse kernels were given by fully-connected neural networks. We train $K$ different reverse kernels $\{l_{k}\}_{k=0}^{K-1}$ for the $K$ transitions, each given by a separate neural network, and amortized over the observation $x$, similarly to \cite{salimans:kingma:welling:2015, huang:tan:lacoste:courville:2018}.
Given the parameters $(\theta, \phi)$, we train these kernels for a large number of epochs using the SIS objective \eqref{eq:elbo_sis} and the Adam optimizer \cite{kingma2014adam}.
  \begin{figure}[!ht]
     \centering
\begin{tabular}{cc}
   \includegraphics[width=.49\linewidth]{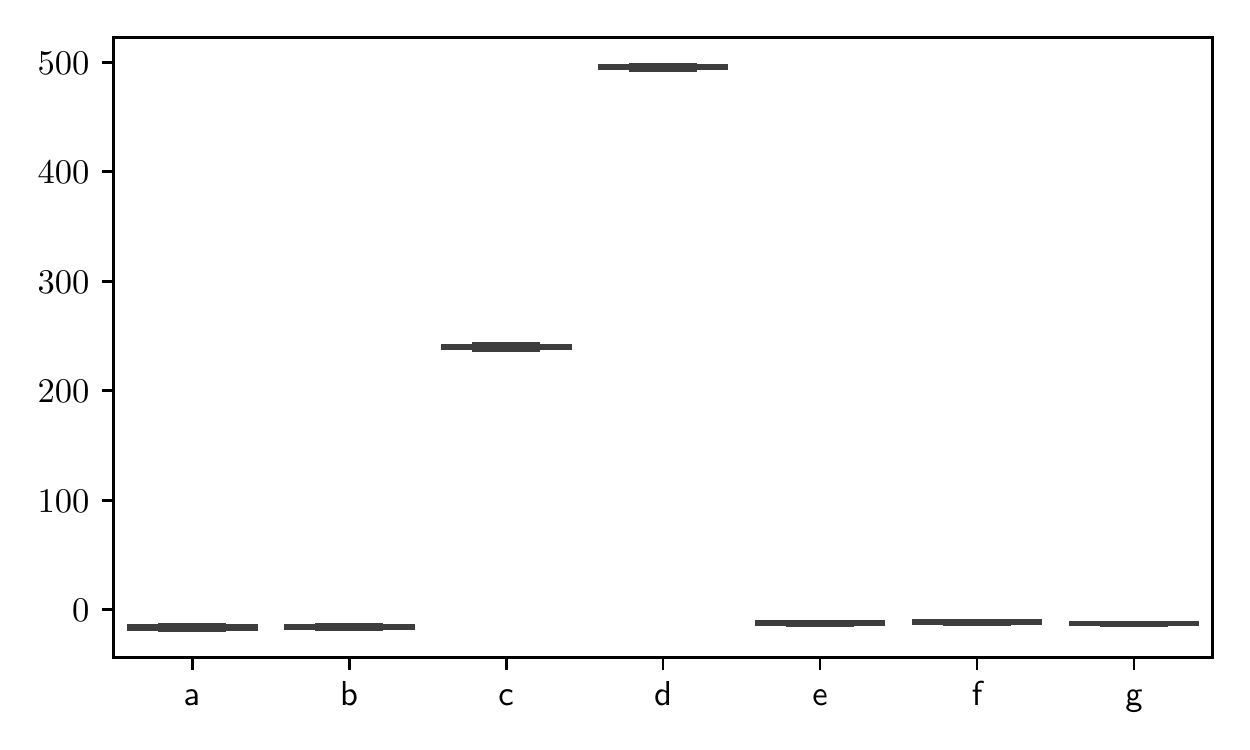}
   & \includegraphics[width=.49\linewidth]{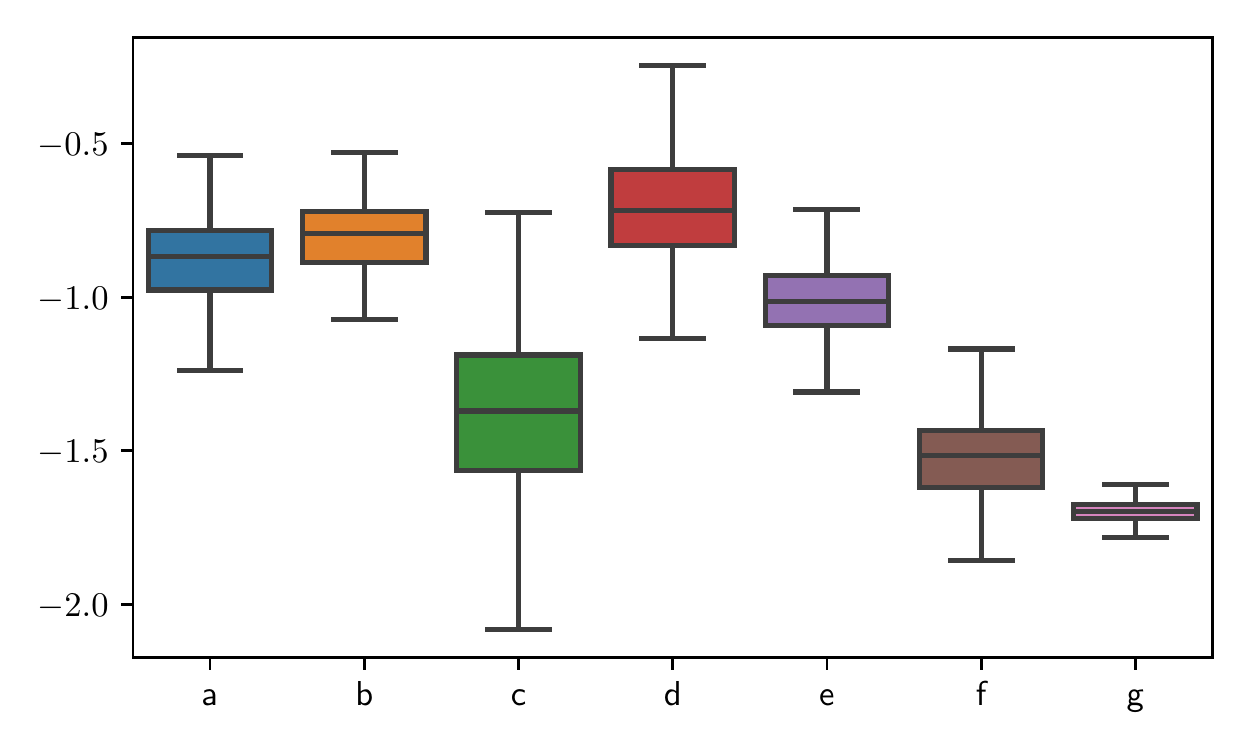}
\end{tabular}
     \caption{Representation of the different estimators (left) and their gradient (right) of the true log likelihood. From left to right, a/ L-MCVAE, $K=5$, b/ L-MCVAE, $K=10$, c/ L-MCVAE, $K=1$, learnable reverse, d/ L-MCVAE, $K=2$ learnable reverse, e/ A-MCVAE, $K=5$, f/ A-MCVAE, $K=10$, g/ A-MCVAE, $K=5$ with control variates.}
     \label{fig:rev_kernel_learnable}
 \end{figure}
In particular, we display in \Cref{fig:rev_kernel_learnable} the different estimators to be compared.
It is easily seen that reverse kernels can not provide reasonable and stable density estimates.
At the same time, we observe the variance of the gradient is higher in those models than in the ones we present in the main text. This motivates our approach bypassing the optimization of the reverse kernels.
\subsection{Additional experimental results}
\label{spsec:exps_celeb}
We display in this section the full results on MNIST, CelebA  and CIFAR respectively of the different models as well as the effect of the different annealing schemes  (respectively in \Cref{tab:MNIST_2}, \Cref{tab:CelebA_2}  and \ref{tab:CIFAR_2}).
 \begin{table*}[]
    \centering
    \caption{Results of the different models on MNIST with different annealing schemes.}
    \label{tab:MNIST_2}
\begin{tabular}{l|l|l|l|l|l|l}
\toprule
                      number of epoches &             \emph{ELBO: }  10 &               30 &              100 &           \emph{NLL: }    10 &               30 &              100 \\
\midrule
                                    VAE &  95.26 $\pm$ 0.5 & 91.58 $\pm$ 0.27 &  89.7 $\pm$ 0.19 & 89.83 $\pm$ 0.59 & 86.86 $\pm$ 0.26 & 85.22 $\pm$ 0.07 \\
                             IWAE, K= 10 & 91.42 $\pm$ 0.21 & 88.56 $\pm$ 0.07 & 87.17 $\pm$ 0.19 & 88.54 $\pm$ 0.27 &  86.07 $\pm$ 0.1 &  84.82 $\pm$ 0.1 \\
                             IWAE, K= 50 & 90.34 $\pm$ 0.27 &  87.5 $\pm$ 0.16 & 86.05 $\pm$ 0.11 &  89.4 $\pm$ 0.25 & 86.54 $\pm$ 0.15 &  85.05 $\pm$ 0.1 \\
                             \hline
                   L-MCVAE Fixed, K= 5 &  96.6 $\pm$ 3.51 &  88.8 $\pm$ 0.46 & 87.77 $\pm$ 0.12 & 90.63 $\pm$ 2.19 & 85.85 $\pm$ 0.27 & 85.07 $\pm$ 0.04 \\
               L-MCVAE Sigmoidal, K= 5 & 95.48 $\pm$ 2.29 & 88.87 $\pm$ 0.82 & 87.81 $\pm$ 0.53 & 90.05 $\pm$ 1.63 & 85.92 $\pm$ 0.62 & 85.16 $\pm$ 0.38 \\
           L-MCVAE All learnable, K= 5 & 96.62 $\pm$ 3.24 & 88.58 $\pm$ 0.75 & 87.51 $\pm$ 0.41 & 90.59 $\pm$ 2.01 & 85.68 $\pm$ 0.49 & 84.92 $\pm$ 0.24 \\
                  L-MCVAE Fixed, K= 10 & 95.98 $\pm$ 3.91 &  88.36 $\pm$ 0.7 & 87.38 $\pm$ 0.35 &  90.5 $\pm$ 2.23 & 85.75 $\pm$ 0.33 &  85.0 $\pm$ 0.11 \\
              L-MCVAE Sigmoidal, K= 10 & 96.78 $\pm$ 0.47 & 88.35 $\pm$ 0.63 & 87.17 $\pm$ 0.52 & 91.13 $\pm$ 0.27 & 85.72 $\pm$ 0.31 & 84.84 $\pm$ 0.26 \\
          L-MCVAE All learnable, K= 10 & 96.78 $\pm$ 1.06 & 87.99 $\pm$ 0.71 &  86.8 $\pm$ 0.66 & 91.33 $\pm$ 0.61 & 85.47 $\pm$ 0.46 & 84.58 $\pm$ 0.39 \\
          \hline
       A-MCVAE Fixed, K= 3 & 96.21 $\pm$ 3.43 & 88.64 $\pm$ 0.78 & 87.63 $\pm$ 0.42 & 90.42 $\pm$ 2.34 & 85.77 $\pm$ 0.65 & 85.02 $\pm$ 0.37 \\
    A-MCVAE Sigmoidal, K= 3 & 96.59 $\pm$ 2.31 &  88.96 $\pm$ 0.4 & 87.86 $\pm$ 0.06 & 90.85 $\pm$ 1.62 & 85.97 $\pm$ 0.34 &  85.17 $\pm$ 0.1 \\
A-MCVAE All learnable, K= 3 & 95.44 $\pm$ 2.68 & 88.79 $\pm$ 0.63 & 87.78 $\pm$ 0.37 &  89.9 $\pm$ 1.68 & 85.96 $\pm$ 0.59 & 85.23 $\pm$ 0.41 \\
       A-MCVAE Fixed, K= 5 & 95.55 $\pm$ 2.96 & 87.99 $\pm$ 0.57 & 87.03 $\pm$ 0.27 & 90.39 $\pm$ 2.21 &  85.6 $\pm$ 0.67 & 84.84 $\pm$ 0.38 \\
    A-MCVAE Sigmoidal, K= 5 & 96.56 $\pm$ 2.02 & 88.51 $\pm$ 0.31 & 87.46 $\pm$ 0.48 & 91.62 $\pm$ 1.55 & 85.96 $\pm$ 0.06 & 85.15 $\pm$ 0.21 \\
A-MCVAE All learnable, K= 5 & 95.81 $\pm$ 1.72 & 88.11 $\pm$ 0.13 & 87.14 $\pm$ 0.18 & 90.79 $\pm$ 1.14 & 85.71 $\pm$ 0.28 & 84.95 $\pm$ 0.04 \\
\hline
 VAE with RealNVP 
 &  95.23 $\pm$ 0.33 &  91.69 $\pm$ 0.15 & 89.62 $\pm$ 0.17 & 89.98 $\pm$ 0.24 & 86.88 $\pm$ 0.05 &  85.23 $\pm$ 0.18 \\
\bottomrule
\end{tabular}
  \end{table*}
\begin{table*}[]
    \centering
    \caption{Full results of the different models on CelebA. All scores must be added 11400 in this table.}
    \label{tab:CelebA_2}
\begin{tabular}{l|l|l|l|l|l|l}
\toprule
                      number of epoches & \emph{ELBO: }                  10 &                  30 &                 100 &   \emph{NLL: }               10 &                  30 &                 100 \\
\midrule
                                    VAE & 23.78 $\pm$ 1.95 &  17.99 $\pm$ 0.4 & 14.72 $\pm$ 0.16 &  17.35 $\pm$ 1.7 & 12.68 $\pm$ 0.62 & 10.11 $\pm$ 0.32 \\
                             IWAE, K= 10 & 20.59 $\pm$ 0.71 & 15.45 $\pm$ 0.52 &   12.2 $\pm$ 0.3 &  18.25 $\pm$ 0.6 & 13.18 $\pm$ 0.42 & 10.14 $\pm$ 0.31 \\
                             IWAE, K= 50 & 19.05 $\pm$ 0.39 &  13.59 $\pm$ 0.5 & 10.48 $\pm$ 0.89 & 19.08 $\pm$ 0.42 & 13.17 $\pm$ 0.54 & 10.12 $\pm$ 0.86 \\
                             \hline
                   L-MCVAE Fixed, K= 5 & 21.93 $\pm$ 1.34 & 13.12 $\pm$ 1.27 & 12.03 $\pm$ 1.21 & 16.65 $\pm$ 1.55 & 10.12 $\pm$ 1.38 & 9.14 $\pm$ 1.27 \\
               L-MCVAE Sigmoidal, K= 5 & 21.61 $\pm$ 1.48 & 12.72 $\pm$ 0.43 &  11.6 $\pm$ 0.37 & 16.42 $\pm$ 1.47 & 9.62 $\pm$ 0.47 &  8.72 $\pm$ 0.4 \\
           L-MCVAE All learnable, K= 5 & 20.75 $\pm$ 0.65 &  12.99 $\pm$ 0.7 & 11.91 $\pm$ 0.61 & 16.16 $\pm$ 0.93 & 10.01 $\pm$ 0.72 & 9.03 $\pm$ 0.64 \\
                  L-MCVAE Fixed, K= 10 & 21.49 $\pm$ 0.03 & 12.83 $\pm$ 0.57 & 11.76 $\pm$ 0.56 & 17.67 $\pm$ 0.75 &  10.26 $\pm$ 0.9 & 9.24 $\pm$ 0.79 \\
              L-MCVAE Sigmoidal, K= 10 & 19.44 $\pm$ 0.82 & 11.81 $\pm$ 0.45 &   10.7 $\pm$ 0.4 & 15.67 $\pm$ 1.48 &  9.24 $\pm$ 0.8 & 8.24 $\pm$ 0.73 \\
          L-MCVAE All learnable, K= 10 &  20.7 $\pm$ 1.15 & 11.81 $\pm$ 0.34 &  10.6 $\pm$ 0.23 &  17.0 $\pm$ 1.87 & 9.29 $\pm$ 0.73 & 8.26 $\pm$ 0.52 \\
          \hline
       A-MCVAE Fixed, K= 3 &  21.59 $\pm$ 1.5 & 13.94 $\pm$ 0.42 &  12.84 $\pm$ 0.3 & 16.64 $\pm$ 1.37 & 10.98 $\pm$ 0.48 &  9.95 $\pm$ 0.3 \\
    A-MCVAE Sigmoidal, K= 3 & 23.63 $\pm$ 1.19 & 14.17 $\pm$ 0.26 & 12.96 $\pm$ 0.18 &  18.0 $\pm$ 0.54 &  11.09 $\pm$ 0.2 & 10.11 $\pm$ 0.13 \\
A-MCVAE All learnable, K= 3 & 22.11 $\pm$ 1.66 & 14.62 $\pm$ 0.35 & 13.54 $\pm$ 0.18 & 17.38 $\pm$ 1.54 & 11.68 $\pm$ 0.33 & 10.67 $\pm$ 0.16 \\
       A-MCVAE Fixed, K= 5 & 20.13 $\pm$ 1.11 & 13.11 $\pm$ 0.38 & 11.99 $\pm$ 0.56 & 16.71 $\pm$ 1.47 & 10.64 $\pm$ 0.24 & 9.63 $\pm$ 0.32 \\
    A-MCVAE Sigmoidal, K= 5 & 20.95 $\pm$ 1.18 & 12.42 $\pm$ 0.42 & 11.13 $\pm$ 0.37 & 17.42 $\pm$ 1.49 & 9.97 $\pm$ 0.59 & 8.82 $\pm$ 0.57 \\
A-MCVAE All learnable, K= 5 & 22.17 $\pm$ 0.17 & 12.73 $\pm$ 0.09 & 11.46 $\pm$ 0.15 & 18.97 $\pm$ 1.04 & 10.41 $\pm$ 0.28 & 9.22 $\pm$ 0.16 \\
\hline
 VAE with RealNVP 
 &  15.56 $\pm$ 0.29 &  13.60 $\pm$ 0.35 & 12.21 $\pm$ 0.27 & 10.69 $\pm$ 0.19 & 9.09 $\pm$ 0.26 &  8.98 $\pm$ 0.2 \\
\bottomrule
\end{tabular}
  \end{table*}
  \begin{table*}[]
    \centering
    \caption{Results of the different models on CIFAR-10 with different annealing schemes. All scores must be added 2800 in this table.}
    \label{tab:CIFAR_2}
\begin{tabular}{l|l|l|l|l|l|l}
\toprule
                      number of epoches &             \emph{ELBO: }    10 &                 30 &                100 &       \emph{NLL: }          10 &                 30 &                100 \\
\midrule
                                    VAE & 69.57 $\pm$ 0.08 & 69.55 $\pm$ 0.51 & 68.84 $\pm$ 0.06 & 68.51 $\pm$ 0.07 & 68.41 $\pm$ 0.33 &  67.9 $\pm$ 0.03 \\
                             IWAE, K= 10 & 69.82 $\pm$ 0.03 & 69.35 $\pm$ 0.03 & 69.36 $\pm$ 0.36 & 68.56 $\pm$ 0.03 &  68.0 $\pm$ 0.03 &  68.02 $\pm$ 0.4 \\
                             IWAE, K= 50 & 69.94 $\pm$ 0.08 & 69.55 $\pm$ 0.04 & 69.43 $\pm$ 0.03 & 69.15 $\pm$ 0.15 & 68.37 $\pm$ 0.18 & 67.93 $\pm$ 0.02 \\
                             \hline
                   L-MCVAE Fixed, K= 5 & 70.86 $\pm$ 0.53 & 68.44 $\pm$ 0.18 & 68.12 $\pm$ 0.11 & 69.37 $\pm$ 0.37 &  67.78 $\pm$ 0.1 & 67.53 $\pm$ 0.07 \\
               L-MCVAE Sigmoidal, K= 5 &  70.9 $\pm$ 0.59 & 68.46 $\pm$ 0.13 & 68.12 $\pm$ 0.11 & 69.42 $\pm$ 0.39 & 67.77 $\pm$ 0.11 & 67.51 $\pm$ 0.08 \\
           L-MCVAE All learnable, K= 5 & 70.62 $\pm$ 0.41 & 68.55 $\pm$ 0.18 &  68.09 $\pm$ 0.1 & 69.15 $\pm$ 0.38 & 67.73 $\pm$ 0.07 &  67.5 $\pm$ 0.07 \\
                  L-MCVAE Fixed, K= 10 & 70.67 $\pm$ 0.42 & 68.37 $\pm$ 0.06 & 69.07 $\pm$ 1.49 & 69.62 $\pm$ 0.54 & 67.78 $\pm$ 0.06 & 67.51 $\pm$ 0.03 \\
              L-MCVAE Sigmoidal, K= 10 & 70.99 $\pm$ 0.59 & 68.36 $\pm$ 0.04 &  68.03 $\pm$ 0.0 &  69.8 $\pm$ 0.67 & 67.76 $\pm$ 0.04 & 67.51 $\pm$ 0.03 \\
          L-MCVAE All learnable, K= 10 & 71.19 $\pm$ 0.79 & 68.36 $\pm$ 0.03 & 68.01 $\pm$ 0.04 & 69.95 $\pm$ 0.62 & 67.78 $\pm$ 0.07 &  67.5 $\pm$ 0.05 \\
          \hline
       A-MCVAE Fixed, K= 3 & 69.97 $\pm$ 0.99 & 68.48 $\pm$ 0.29 & 68.18 $\pm$ 0.16 & 69.26 $\pm$ 0.76 & 67.77 $\pm$ 0.18 &  67.55 $\pm$ 0.1 \\
    A-MCVAE Sigmoidal, K= 3 &  70.5 $\pm$ 1.18 & 68.45 $\pm$ 0.28 & 68.19 $\pm$ 0.18 &  69.18 $\pm$ 0.8 & 67.77 $\pm$ 0.19 & 67.56 $\pm$ 0.11 \\
A-MCVAE All learnable, K= 3 & 70.69 $\pm$ 1.23 &  68.44 $\pm$ 0.3 & 68.17 $\pm$ 0.18 & 69.36 $\pm$ 0.89 &  67.76 $\pm$ 0.2 & 67.55 $\pm$ 0.11 \\
       A-MCVAE Fixed, K= 5 & 70.37 $\pm$ 1.04 & 68.31 $\pm$ 0.21 &  68.04 $\pm$ 0.1 & 69.36 $\pm$ 0.87 & 67.73 $\pm$ 0.17 & 67.51 $\pm$ 0.08 \\
    A-MCVAE Sigmoidal, K= 5 & 70.89 $\pm$ 0.38 &  68.4 $\pm$ 0.05 & 68.07 $\pm$ 0.04 & 69.71 $\pm$ 0.33 &  67.8 $\pm$ 0.04 & 67.53 $\pm$ 0.02 \\
A-MCVAE All learnable, K= 5 &  70.1 $\pm$ 0.89 &  68.28 $\pm$ 0.2 & 68.01 $\pm$ 0.08 & 69.23 $\pm$ 0.75 & 67.71 $\pm$ 0.15 &  67.5 $\pm$ 0.07 \\
\hline
 VAE with RealNVP 
 &  70.01 $\pm$ 0.12 &  69.51 $\pm$ 0.07 & 69.19 $\pm$ 0.13 & 68.73 $\pm$ 0.05 & 68.35 $\pm$ 0.05 &  68.05 $\pm$ 0.02 \\
\bottomrule
\end{tabular}
  \end{table*}
\section{Proofs}
\label{spsec:proofs}
\subsection{Proof of SIS and AIS Identities}
\begin{proposition}
Let $\{ \Gamma_k \}_{k=0}^K$ be a sequence of distributions on $(\rset^d, \mcb{\rset^d})$, $\{M_k\}_{k=1}^K$ and $\{ L_k \}_{k=0}^{K-1}$ be Markov kernels. Assume that for each $k \in \{0,\dots,K-1\}$, there exists a positive measurable function $w_k : \rset^d \times \rset^d \ \mapsto \rset_+$ such that
\begin{equation}
\label{eq:condition}
\Gamma_k(\rmd z_k) L_{k-1}(z_k, \rmd z_{k-1}) = \Gamma_{k-1}(\rmd z_{k-1}) M_k(z_{k-1}, \rmd z_k) w_k(z_{k-1},z_k)  \eqsp.
\end{equation}
Then,
\begin{equation}
\label{eq:smc_identity}
\Gamma_0(\rmd z_0) \prod_{k=1}^K M_k(z_{k-1},\rmd z_k) \prod_{k=1}^K w_k(z_{k-1},z_k)=   \Gamma_K(\rmd z_K) \prod_{k=K}^{1} L_{k-1}(z_k, \rmd z_{k-1})\eqsp.
\end{equation}
\end{proposition}
\begin{proof}
We prove by induction that for $k \in \{1,\dots,K\}$,
\begin{equation}
\label{eq:induction}
\Gamma_0(\rmd z_0) \prod_{i=1}^k M_i(z_{i-1},\rmd z_i) \prod_{i=1}^k w_i(z_{i-1},z_i)=   \Gamma_k(\rmd z_k) \prod_{i=k}^{1} L_{i-1}(z_i, \rmd z_{i-1})\eqsp.
\end{equation}
Eq. \eqref{eq:induction} is satisfied for $k=1$ by \eqref{eq:condition}. Assume that \eqref{eq:induction} is satisfied for $k \leq K-1$. By \eqref{eq:condition},
\begin{align*}
\Gamma_{k+1}(\rmd z_{k+1}) \prod_{i=k+1}^{1} L_{i-1}(z_i, \rmd z_{i-1})
&= \Gamma_{k+1}(\rmd z_{k+1}) L_{k}(z_{k+1}, \rmd z_{k}) \prod_{i=k}^{1} L_{i-1}(z_i, \rmd z_{i-1}) \\
&=  \Gamma_{k}(\rmd z_{k}) M_{k+1}(z_{k}, \rmd z_{k+1}) w_{k+1}(z_{k},z_{k+1})  \prod_{i=k}^{1} L_{i-1}(z_i, \rmd z_{i-1}) \\
&=  M_{k+1}(z_{k}, \rmd z_{k+1}) w_{k+1}(z_{k},z_{k+1}) \Gamma_0(\rmd z_0) \prod_{i=1}^k M_i(z_{i-1},\rmd z_i) \prod_{i=1}^k w_i(z_{i-1},z_i)
\end{align*}
which concludes the proof.
\end{proof}
We now highlight conditions under which \eqref{eq:condition} is satisfied.
\begin{enumerate}
\item Assume that $\{ \Gamma_k \}_{k=0}^K$ have positive densities \wrt\ to the Lebesgue measure, \ie\ $\Gamma_k(\rmd z_k) = \Gamma_k(z_k ) \rmd z_k$ and that the kernels $\{ M_k \}_{k=1}^K$ and $\{ L_k \}_{k=0}^{K-1}$ have positive transition densities
$M_k(z_{k-1},\rmd z_k)= m_k(z_{k-1},z_k) \rmd z_k$ and $L_{k-1}(z_k,\rmd z_{k-1})= \ell_{k-1}(z_{k},z_{k-1}) \rmd z_{k-1}$, $k \in \{1,\dots,K\}$. Then,
\[
w_k(z_{k-1},z_k)= \frac{\gamma_k(z_k) \ell_{k-1}(z_k, z_{k-1})}{\gamma_{k-1}(z_{k-1}) m_k(z_{k-1}, z_k)}
\]
\item Assume that for $k \in \{1,\dots,K\}$,  $\Gamma_k(\rmd z_{k-1}) M_k(z_{k-1},\rmd z_k)= \Gamma_k(\rmd z_k) L_{k-1}(z_k, \rmd z_{k-1})$, and that there exists a positive measurable function such that $\Gamma_k(\rmd z_{k-1}) = \tilde{w}_k(z_{k-1}) \Gamma_{k-1}(\rmd z_{k-1})$. Then,
\begin{align*}
\Gamma_k(\rmd z_k) L_{k-1}(z_k, \rmd z_{k-1})= \Gamma_k(\rmd z_{k-1}) M_k(z_{k-1}, \rmd z_{k})= \tilde{w}_k(z_{k-1}) \Gamma_{k-1}(\rmd z_{k-1}) M_k(z_{k-1},\rmd z_k) \eqsp.
\end{align*}
Hence, \eqref{eq:condition} is satisfied with $w_k(z_{k-1},z_k)= \tilde{w}_k(z_{k-1})$. In particular, if for all $k \in \{0,\dots,K\}$, $\Gamma_k(z_k)= \gamma_k(z_k) \rmd z_k$, where $\gamma_k$ is a positive p.d.f., then $\tilde{w}_k(z_k)= \gamma_k(z_k)/ \gamma_{k-1}(z_{k-1})$.
\item Assume that for $k \in \{1,\dots,K\}$, $M_k$ is reversible \wrt\ $\Gamma_k$, \ie\ $\Gamma_k(\rmd z_{k-1}) M_k(z_{k-1},\rmd z_k)= \Gamma_k(\rmd z_k) M_k(z_k, \rmd z_{k-1})$, and that there exists a positive measurable function such that $\Gamma_k(\rmd z_{k-1}) = \tilde{w}_k(z_{k-1}) \Gamma_{k-1}(\rmd z_{k-1})$. Then, setting $L_{k-1}= M_k$, \eqref{eq:condition} is satisfied.
\end{enumerate}

\subsection{Proof of \eqref{eq:elbo_sis}}

For $k\in\{1,\dots, K\}$, $z_{k-1}\in\rset^d$, denote by $G_{k, z_{k-1}}$ the mapping $u_k\mapsto \propmap{u_k}(z_{k-1})$.
Our derivation below rely on the fact that for $k\in\{1,\dots, K\}$, $z_{k-1}\in\rset^d$, $G_{k, z_{k-1}}$ is a $\rmc^1$-diffeomorphism. This is the case for the Langevin mappings.
Note, similarly to the density considered in \Cref{sec:ais}, that $m_k(z_{k-1}, z_k) = \varphi(G_{k, z_{k-1}}^{-1}(z_k)) \Jac_{G_{k, z_{k-1}}^{-1}}(z_k)$.
When $K=1$, we have
  \begin{align*}
  \nonumber
    \int \log \bigl(w_{1} (z_{0}, z_1)\bigr) \initdistr^1(\chunk{z}{0}{1}\mid x) \rmd \chunk{z}{0}{1} &= \int \log \bigl(w_{1} (z_{0}, z_1)\bigr) \initdistr(z_0\mid x)\Jac_{G^{-1}_{1, z_{0}}}(z_1) \varphi\bigl(G^{-1}_{1, z_{0}}(z_1)\bigr)  \rmd \chunk{z}{0}{1}\\
    &= \int \log \bigl(w_{1} (z_{0}, \propmap{1, u_1}(z_0))\bigr) \initdistr(z_0 \mid x) \varphi(u_1) \rmd z_0 \rmd u_1\eqsp,
  \end{align*}
  where we have performed the change of variables $u_1 = G^{-1}_{1, z_{0}}(z_1)$, hence $z_1 = G_{1, z_{0}}(u_1) = \propmap{1, u_1}(z_0)$. Let now $K$ be in $\nsets$.
In general, we write
\begin{align*}
    \elbosmc &=   \int_{} \log \left(\prod_{k = 1}^K w_{k} (z_{k - 1}, z_k)\right) \initdistr^K(\chunk{z}{0}{K}\mid x) \rmd \chunk{z}{0}{K} =  \int_{} \log \left(\prod_{k = 1}^K w_{k} (z_{k - 1}, z_k)\right) \initdistr(z_0\mid x) \prod_{k=1}^K m_k(z_{k-1}, z_k)\rmd \chunk{z}{0}{K-1} \rmd z_K\\
    &= \int_{} \log \left(\prod_{k = 1}^K w_{k} (z_{k - 1}, z_k)\right) \initdistr(z_0\mid x) \prod_{k=1}^{K-1} m_k(z_{k-1}, z_k)\varphi(G_{K, z_{K-1}}^{-1}(z_K)) \Jac_{G_{K, z_{K-1}}^{-1}}(z_K)\rmd \chunk{z}{0}{K-1} \rmd z_K \\
    &=  \int_{} \log \left(\prod_{k = 1}^{K-1} w_{k} (z_{k - 1}, z_k)w_{K} (z_{K - 1}, \propmap{u_K}(z_{K-1}))\initdistr(z_0\mid x)\right) \prod_{k=1}^{K-1} m_k(z_{k-1}, z_k) \varphi(u_K)\rmd \chunk{z}{0}{K-1}\rmd u_K
\end{align*}
using the change of variables $u_K =G_{K, z_{K-1}}^{-1}(z_K) $.
By an immediate backwards induction, we write
\begin{equation*}
    \elbosmc = \int_{} \log\left( \prod_{k=1}^{K}  w_{k}(\Circ{i = 1}{k-1} \propmap{i, u_i} (z_0),\Circ{i = 1}{k} \propmap{i, u_i} (z_0))\right)\initdistr(z_0\mid x) \varphi(\chunk{u}{1}{K})\rmd z_0\rmd \chunk{u}{1}{K}\eqsp.
\end{equation*}
\subsection{Proof of \Cref{lem:langevin_map_diff}}
    Let $\eta <\lipcon^{-1}$ and  $u \in \rset^D$.
    First we show that $\propmap{u}^{\MALA}$ is invertible.
    Consider, for each $(y,u) \in \rset^{2d}$, the mapping $H_{y,u}(z) = y - \sqrt{2\eta} u - \eta \nabla \log\pi(z)$.
    We have, for $z_1, z_2 \in \rset^d$,
    \begin{equation*}
    \label{eq:H_contracting}
      \|H_{y,u}(z_1) - H_{y,u}(z_2)\| \leq  \eta \|\nabla \log\pi(z_1) - \nabla \log\pi(z_2)\| \leq \eta \lipcon \|z_1-z_2\|
    \end{equation*}
    and $ \eta \lipcon < 1$. Hence $H_{y,u}$ is a contraction mapping and thus has a unique fixed point $z_{y,u}$. Hence, for all $(y,u) \in \rset^{2d}$ there exists a unique $z_{y,u}$ satisfying
    \begin{equation*}
    \label{eq:Invertibility_MALA}
      H_{y,u}(z_{y,u}) = z_{y,u} \Rightarrow y = z_{y,u} + \eta \nabla\log \pi(z_{y,u}) + \sqrt{2\eta}u = \propmap{u}^{\MALA}(z_{y,u}).
    \end{equation*}
  This establishes the  invertibility of $\propmap{u}^{\MALA}$. The fact that the inverse of $\propmap{u}^{\MALA}$ is $\rmC^1$ follows from a simple application of the local inverse function theorem.

\section{ELBO AIS}
\label{spsec:ais_elbo}
\subsection{Construction of the control variates}
We prove in this section that the variance reduced objective we consider is valid.
Sample now $n$ samples $\chunk{u}{0}{K}^{1:n}\overset{\textup{i.i.d.}}{\sim} \densinnovationkp$.
For an index $i\in \{1,\dots, n\}$, given the initial point $z_0^i= V_{\phi,x}(u_0^i)$ and the innovation noise $\chunk{u}{1}{K}^i$, we sample the A/R booleans $\chunk{a}{1}{K}^i$.
We introduce, in the main text, for $i\in\{1,\dots, n\}$
  \[
   \tilde W_{n,i} = \frac{1}{n-1}\sum_{j\neq i} W(V_{\phi,x}(u_0^j), \chunk{a^j}{1}{K}, \chunk{u^j}{1}{K})\eqsp,
  \]
  $\tilde W_i$ provides a reasonable estimate of the AIS ELBO but is independent from the $i$-th trajectory.
  We use this quantity as a control variate to reduce the variance of our gradient estimator by introducing
  \begin{align}
  \label{eq:grad_elbo_ais_cv}
  \nonumber
    \widehat{\nabla \elboais}_n &= n^{-1} \sum_{i=1}^n \nabla W(V_{\phi,x}(u_0^i), \chunk{a^i}{1}{K}, \chunk{u^i}{1}{K})
    \\ \nonumber
      &+ n^{-1}\sum_{i=1}^n\left[W(V_{\phi,x}(u_0^i), \chunk{a^i}{1}{K}, \chunk{u^i}{1}{K}) -   \tilde{W}_{n,i} \right]
      \\
      &\times\nabla \log A(V_{\phi,x}(u_0^i), \chunk{a^i}{1}{K}, \chunk{u^i}{1}{K}) \eqsp.
  \end{align}
Proving its unbiasedness boils down to proving that the term
$n^{-1}\sum_{i=1}^n   \tilde{W}_{n,i}\nabla  \log A(V_{\phi,x}(u_0^i), \chunk{a^i}{1}{K}, \chunk{u^i}{1}{K})
$
has expectation zero.
Let us compute for $i\in\{1\dots, n\}$,
\begin{multline*}
   \int_{} \sum_{\chunk{a}{1}{K}^i} \densinnovationkp(\chunk{u}{0}{K}^i)  A(V_{\phi,x}(u_0^i), \chunk{a^i}{1}{K}, \chunk{u^i}{1}{K})\tilde{W}_{n,i}\nabla  \log A(V_{\phi,x}(u_0^i), \chunk{a^i}{1}{K}, \chunk{u^i}{1}{K}) \rmd \chunk{u}{0}{K}^i = \\  \int_{}\sum_{\chunk{a}{1}{K-1}^i}  \densinnovationkp(\chunk{u}{0}{K}^i) \prod_{k=1}^{K}\alpha^{a^i_k}_{k, u^i_k} \bigl(z_{k-1}^i\bigr)\tilde{W}_{n,i}\left[\nabla \sum_{k=1}^{K-1} \log \alpha^{a^i_k}_{k, u^i_k} \bigl(z_{k-1}^i\bigr)+ \sum_{a_K^i}\nabla\log \alpha^{a^i_K}_{K, u^i_K} \bigl(z_{K-1}^i\bigr)\right] \rmd \chunk{u}{0}{K}^i\eqsp,
\end{multline*}
denoting $z_0^j = V_{\phi, x}(u_0^j)$,  $ z_k^j = \Circtext{i = 1}{k} \propmap{i, u^j_i}^{a^j_i} (z_0^{j})$ by simplicity of notation.
Yet, $\sum_{a_K^i} \alpha^{a^i_K}_{K, u^i_K} \bigl(z_{K-1}^i\bigr) = 1$ exactly, thus  $\sum_{a_K^i}\alpha^{a^i_K}_{K, u^i_K} \bigl(z_{K-1}^l\bigr)\nabla\log \alpha^{a^i_K}_{K, u^i_K} \bigl(z_{K-1}^i\bigr)= 0$.
We can thus show by an immediate induction that
$\int_{} \sum_{\chunk{a}{1}{K}^i}  \densinnovation(\chunk{u}{0}{K}^i) \tilde{W}_{n,i}\nabla  \log A(V_{\phi,x}(u_0^i), \chunk{a^i}{1}{K}, \chunk{u^i}{1}{K})\rmd \chunk{u}{i}{K}^i  = 0$, as $\tilde{W}_{n,i}$ is a constant in that integral by independence of the samples for $i\in\{1\dots, n\}$.
Moreover, as
\begin{multline*}
  \int_{} \sum_{\chunk{a}{1}{K}^{1:n}}\sum_{i=1}^n   \tilde{W}_{n,i}\nabla  \log A(V_{\phi,x}(u_0^i), \chunk{a^i}{1}{K}, \chunk{u^i}{1}{K}) \prod_{\ell=1}^n \densinnovationkp(\chunk{u}{0}{K}^\ell)\rmd \chunk{u}{0}{K}^{1:n} = \\\int_{} \sum_{i=1}^n\sum_{\chunk{a}{1}{K}^{-i}}\left[\int\sum_{\chunk{a}{1}{K}^{i}}  \tilde{W}_{n,i} \nabla \log A(V_{\phi,x}(u_0^i), \chunk{a^i}{1}{K}, \chunk{u^i}{1}{K})\densinnovationkp(\chunk{u}{0}{K}^i)\rmd \chunk{u}{0}{K}^{i}\right]\prod_{\ell\neq i} \densinnovationkp(\chunk{u}{0}{K}^\ell)\rmd \chunk{u}{1}{K}^{-i}\eqsp,
\end{multline*}
then $n^{-1}\sum_{i=1}^n   \tilde{W}_{n,i}\nabla  \log A(V_{\phi,x}(u_0^i), \chunk{a^i}{1}{K}, \chunk{u^i}{1}{K})
$ is of zero expectation, and \eqref{eq:grad_elbo_ais_cv} is an unbiased estimator of the gradient.

\subsection{Discussion of \cite{wunoe2020stochastic}}

In \cite{wunoe2020stochastic}, authors consider a MCMC VAE inspired by AIS. The model used however is quite different in spirit to what is performed in this work.
\cite{wunoe2020stochastic} use Langevin mappings and accept reject steps in their VAE.
Note however that the A/R probabilities defined are written as
$$\alpha(x, y) = 1\wedge \pi(y)/\pi(x)\eqsp, $$
different from \eqref{eq:mala_acceptance}.
Moreover, even though accept/reject steps are considered, the score function estimator \eqref{eq:reinforce_ais} is not taken into account.

Finally, the initial density of the sequence is not taken to be some variational mean field initialization but directly the prior in the latent space. As a result, the scores obtained by the MCMC VAE are less competitive than that of the RNVP VAE presented in \citep[Table 3.]{wunoe2020stochastic}, contrary to what is presented here.
\end{document}